%% file: sparse_nnls.tex
\documentclass[11pt]{article}

\input{macros}

\begin{document}

\title{
{\bf A Method for Finding Structured Sparse Solutions\\ to Non-negative Least Squares Problems \\with Applications}}
\renewcommand{\baselinestretch}{1.0}
\date{}
\maketitle

\thispagestyle{empty}
{\large \bf \begin{center}
\hspace{.1 in}\\
\hspace{.1 in}\\
Ernie Esser$^1$, \ \ Yifei Lou$^1$, \ \ Jack Xin\footnote{
Department of Mathematics, UC Irvine, Irvine, CA 92697. Author emails: eesser@uci.edu, \\
louyifei@gmail.com, jxin@math.uci.edu.
The work was partially supported by NSF grants DMS-0911277, DMS-0928427, and DMS-1222507.}
\end{center}}
\bigskip

\input{abstract}
\newpage

\input{introduction}

\input{problem}

\input{algorithm}

\input{applications}

\input{numerics}

\input{conclusions}

\bibliographystyle{IEEEbib}
\clearpage
\bibliography{sparse_nnls_refs}

\end{document}

%% file: macros.tex
\usepackage[margin=3cm]{geometry}
\usepackage{epsfig}
\usepackage{amsmath}
\usepackage{amssymb}
\usepackage{amsthm}
\usepackage{needspace}

\newcommand{\R}{\mathbb{R}}

\newcommand{\I}{\mathrm{I}}

\DeclareMathOperator{\diag}{diag}

\theoremstyle{plain}
\newtheorem{theorem}{Theorem}[section]
\newtheorem{lemma}[theorem]{Lemma}
\newtheorem{proposition}[theorem]{Proposition}
\newtheorem*{corollary}{Corollary}

\theoremstyle{definition}

\newtheorem{assumption}{Assumption}

\newcommand{\bbm}{\begin{bmatrix}}
\newcommand{\ebm}{\end{bmatrix}}

\newcommand{\bi}{\begin{itemize}}
\newcommand{\ei}{\end{itemize}}

\newcommand{\ben}{\begin{enumerate}}
\newcommand{\een}{\end{enumerate}}

\newcommand{\bx}{\mathbf x}            
\newcommand{\by}{\mathbf y}            

\newcounter{alg}
\newcommand{\alg}[1]{\refstepcounter{alg} \vspace{.2in} \needspace{5 \baselineskip}
\hrule \vspace{.1in} \noindent Algorithm \thealg : {#1} \vspace{.1in} \hrule \vspace{.2in} \noindent}

\newcommand{\comments}[1]{}

\newenvironment{changemargin}[2]{%
\list{}{\rightmargin#2\leftmargin#1
\parsep=0pt\topsep=0pt\partopsep=0pt}
\item[]}
{\endlist}

%% file: abstract.tex
\begin{abstract}
Demixing problems in many areas such as hyperspectral imaging and differential optical absorption spectroscopy (DOAS) often require finding sparse nonnegative linear combinations of dictionary elements that match observed data.  We show how aspects of these problems, such as misalignment of DOAS references and uncertainty in hyperspectral endmembers, can be modeled by expanding the dictionary with grouped elements and imposing a structured sparsity assumption that the combinations within each group should be sparse or even 1-sparse.  If the dictionary is highly coherent, it is difficult to obtain good solutions using convex or greedy methods, such as non-negative least squares (NNLS) or orthogonal matching pursuit.  We use penalties related to the Hoyer measure, which is the ratio of the $l_1$ and $l_2$ norms, as sparsity penalties to be added to the objective in NNLS-type models.  For solving the resulting nonconvex models, we propose a scaled gradient projection algorithm that requires solving a sequence of strongly convex quadratic programs.  We discuss its close connections to convex splitting methods and difference of convex programming.  We also present promising numerical results for example DOAS analysis and hyperspectral demixing problems.

\end{abstract} 

%% file: introduction.tex
\section{Introduction \label{introduction}}

A general demixing problem is to estimate the quantities or concentrations of the individual components of some observed mixture.  Often a linear mixture model is assumed \cite{KM}.  In this case the observed mixture $b$ is modeled as a linear combination of references for each component known to possibly be in the mixture.  If we put these references in the columns of a dictionary matrix $A$, then the mixing model is simply $Ax = b$.  Physical constraints often mean that $x$ should be nonnegative, and depending on the application we may also be able to make sparsity assumptions about the unknown coefficients $x$.  This can be posed as a basis pursuit problem where we are interested in finding a sparse and perhaps also non-negative linear combination of dictionary elements that match observed data.  This is a very well studied problem.  Some standard convex models are non-negative least squares (NNLS) \cite{LH,SH},
\begin{equation}\label{eq:NNLS} \min_{x \geq 0} \frac{1}{2}\|Ax-b\|^2 \end{equation}
and methods based on $l_1$ minimization \cite{CDS,Tibshirani,YOGD}.  There are also variations that enforce different group sparsity assumptions on $x$ \cite{MGB,JAB,QG}.

In this paper we are interested in how to deal with uncertainty in the dictionary. The case when the dictionary is unknown is dealt with in sparse coding and non-negative matrix factorization (NMF) problems \cite{Pauca06,MBPS,H2002,LS,BBLPP,EMOSX}, which require learning both the dictionary and a sparse representation of the data.  We are, however, interested in the case where we know the dictionary but are uncertain about each element.  One example we will study in this paper is differential optical absorption spectroscopy (DOAS) analysis \cite{PS}, for which we know the reference spectra but are uncertain about how to align them with the data because of wavelength misalignment.  Another example we will consider is hyperspectral unmixing \cite{unmixing,Greer10,GWO}.  Multiple reference spectral signatures, or endmembers, may have been measured for the same material, and they may all be slightly different if they were measured under different conditions.  We may not know ahead of time which one to choose that is most consistent with the measured data.  Although there is previous work that considers noise in the endmembers \cite{HLS} and represents endmembers as random vectors \cite{ZG}, we may not always have a good general model for endmember variability.  For the DOAS example, we do have a good model for the unknown misalignment \cite{PS}, but even so, incorporating it may significantly complicate the overall model.  Therefore for both examples, instead of attempting to model the uncertainty, we propose to expand the dictionary to include a representative group of possible elements for each uncertain element as was done in \cite{louBS11}.

The grouped structure of the expanded dictionary is known by construction, and this allows us to make additional structured sparsity assumptions about the corresponding coefficients.  In particular, the coefficients should be extremely sparse within each group of representative elements, and in many cases we would like them to be at most $1$-sparse.  We will refer to this as intra group sparsity.  If we expected sparsity of the coefficients for the unexpanded dictionary, then this will carry over to an inter group sparsity assumption about the coefficients for the expanded dictionary.  By inter group sparsity we mean that with the coefficients split into groups, the number of groups containing nonzero elements should also be sparse.  Modeling structured sparsity by applying sparsity penalties separately to overlapping subsets of the variables has been considered in a much more general setting in \cite{JAB,JOB}.

The expanded dictionary is usually an underdetermined matrix with the property that it is highly coherent because the added columns tend to be similar to each other.  This makes it very challenging to find good sparse representations of the data using standard convex minimization and greedy optimization methods.  If $A$ satisfies certain properties related to its columns not being too coherent \cite{BEZ}, then sufficiently sparse non-negative solutions are unique and can therefore be found by solving the convex NNLS problem.  These assumptions are usually not satisfied for our expanded dictionaries, and while NNLS may still be useful as an initialization, it does not by itself produce sufficiently sparse solutions. Similarly, our expanded dictionaries usually do not satisfy the incoherence assumptions required for $l_1$ minimization or greedy methods like Orthogonal Matching Pursuit (OMP) to recover the $l_0$ sparse solution \cite{Tropp,CRT}.  However, with an unexpanded dictionary having relatively few columns, these techniques can be effectively used for sparse hyperspectral unmixing \cite{IBP}.

The coherence of our expanded dictionary means we need to use different tools to find good solutions that satisfy our sparsity assumptions.  We would like to use a variational approach as similar as possible to the NNLS model that enforces the additional sparsity while still allowing all the groups to collaborate.  We propose adding nonconvex sparsity penalties to the NNLS objective function (\ref{eq:NNLS}).  We can apply these penalties separately to each group of coefficients to enforce intra group sparsity, and we can simultaneously apply them to the vector of all coefficients to enforce additional inter group sparsity.  From a modeling perspective, the ideal sparsity penalty is $l_0$.  There is a very interesting recent work that directly deals with $l_0$ constraints and penalties via a quadratic penalty approach \cite{luZ12}.  If the variational model is going to be nonconvex, we prefer to work with a differentiable objective when possible.  We therefore explore the effectiveness of sparsity penalties based on the Hoyer measure \cite{H2004,HR}, which is essentially the ratio of $l_1$ and $l_2$ norms.  In previous works, this has been successfully used to model sparsity in NMF and blind deconvolution applications \cite{H2004,KTF,JLSW}.  We also consider the difference of $l_1$ and $l_2$ norms.  By the relationship, $\|x\|_1 - \|x\|_2 = \|x\|_2 (\frac{\|x\|_1}{\|x\|_2} - 1)$, we see that while the ratio of norms is constant in radial directions, the difference increases moving away from the origin except along the axes.  Since the Hoyer measure is twice differentiable on the non-negative orthant away from the origin, it can be locally expressed as a difference of convex functions, and convex splitting or difference of convex (DC) methods \cite{TA} can be used to find a local minimum of the nonconvex problem.  Some care must be taken, however, to deal with its poor behavior near the origin.  It is even easier to apply DC methods when using $l_1$ - $l_2$ as a penalty, since this is already a difference of convex functions, and it is well defined at the origin.

The paper is organized as follows.  In Section \ref{problem} we define the general model, describe the dictionary structure and show how to use both the ratio and the difference of $l_1$ and $l_2$ norms to model our intra and inter group sparsity assumptions.  Section \ref{algorithm} derives a method for solving the general model, discusses connections to existing methods and includes convergence analysis.  In Section \ref{applications} we discuss specific problem formulations for several examples related to DOAS analysis and hyperspectral demixing.
Numerical experiments for comparing methods and applications to example problems are presented in Section \ref{numerics}.



%% file: problem.tex
\section{Problem \label{problem}}

For the non-negative linear mixing model $Ax=b$, let $b \in \R^W$, $A \in R^{W \times N}$ and $x \in \R^N$ with $x \geq 0$.  Let the dictionary $A$ have $l_2$ normalized columns and consist of $M$ groups, each with $m_j$ elements.  We can write $A = \bbm A_1 & \cdots & A_M \ebm$ and $x = {\bbm x_1 & \cdots & x_M \ebm}^T$, where each $x_j \in \R^{m_j}$ and $N = \sum_{j=1}^M m_j$.  The general non-negative least squares problem with sparsity constraints that we will consider is
\begin{equation}\label{eq:generalprob}
\min_{x \geq 0} F(x) := \frac{1}{2}\|Ax-b\|^2 + R(x) \ , \end{equation}
where
\begin{equation}\label{eq:R} R(x) = \sum_{j=1}^M \gamma_j R_j(x_j) + \gamma_0 R_0(x) \ . \end{equation}
The functions $R_j$ represent the intra group sparsity penalties applied to each group of coefficients $x_j$, $j = 1,...,M$, and $R_0$ is the inter group sparsity penalty applied to $x$.  If $F$ is differentiable, then a necessary condition for $x^*$ to be a local minimum is given by
\begin{equation}\label{eq:genopt}
(y - x^*)^T \nabla F(x^*) \geq 0 \qquad \forall y \geq 0 \ .
\end{equation}

For the applications we will consider, we want to constrain each vector $x_j$ to be at most 1-sparse, which is to say that we want $\|x_j\|_0 \leq 1$.  To accomplish this through the model (\ref{eq:generalprob}), we will need to choose the parameters $\gamma_j$ to be sufficiently large.

The sparsity penalties $R_j$ and $R_0$ will either be the ratios of $l_1$ and $l_2$ norms defined by
\begin{equation}\label{H} H_j(x_j) = \gamma_j \frac{\|x_j\|_1}{\|x_j\|_2} \qquad \text{and} \qquad H_0(x) = \gamma_0 \frac{\|x\|_1}{\|x\|_2} \ , \end{equation}
or they will be the differences defined by
\begin{equation}\label{S} S_j(x_j) = \gamma_j (\|x_j\|_1 - \|x_j\|_2) \qquad \text{and} \qquad S_0(x) = \gamma_0 (\|x\|_1 - \|x\|_2) \ . \end{equation}
A geometric intuition for why minimizing $\frac{\|x\|_1}{\|x\|_2}$ promotes sparsity of $x$ is that since it is constant in radial directions, minimizing it tries to reduce $\|x\|_1$ without changing $\|x\|_2$.  As seen in Figure \ref{unitballs}, sparser vectors have smaller $l_1$ norm on the $l_2$ sphere.
\begin{figure}
\begin{center}
\epsfig{file=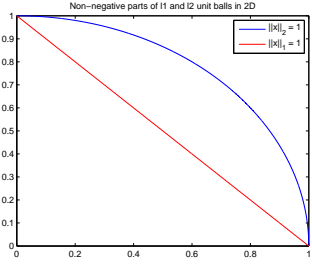,width=5cm,clip=}
\end{center}
\caption{$l_1$ and $l_2$ unit balls \label{unitballs}}
\end{figure}

Neither $H_j$ or $S_j$ is differentiable at zero, and $H_j$ is not even continuous there.  Figure \ref{fig:Rl1l2} shows a visualization of both penalties in two dimensions.
\begin{figure}
\begin{center}
\epsfig{file=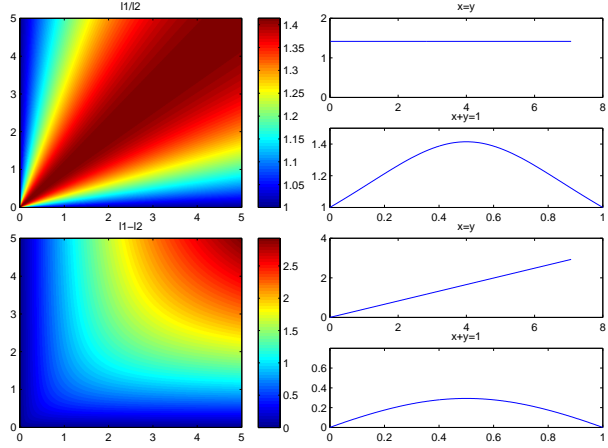,width=10cm,clip=}
\end{center}
\caption{Visualization of $l_1$/$l_2$ and $l_1$ - $l_2$ penalties \label{fig:Rl1l2}}
\end{figure}
To obtain a differentiable $F$, we can smooth the sparsity penalties by replacing the $l_2$ norm with the Huber function, defined by the infimal convolution
\begin{equation}\label{eq:huber}
\phi(x,\epsilon) = \inf_y \|y\|_2 + \frac{1}{2 \epsilon}\|y - x\|^2 = \begin{cases} \frac{\|x\|_2^2}{2 \epsilon} & \text{if } \ \|x\|_2 \leq \epsilon \\ \|x\|_2 - \frac{\epsilon}{2} & \text{otherwise}\ . \end{cases}
\end{equation}
In this way we can define differentiable versions of sparsity penalties $H$ and $S$ by
\begin{align}
H_j^{\epsilon_j}(x_j) & = \gamma_j \frac{\|x_j\|_1}{\phi(x_j,\epsilon_j)+\frac{\epsilon_j}{2}} \label{Hep} \\
\notag H_0^{\epsilon}(x) & = \gamma_0 \frac{\|x\|_1}{\phi(x,\epsilon_0)+\frac{\epsilon_0}{2}} \end{align}
\begin{align}
S_j^{\epsilon}(x_j) & = \gamma_j (\|x_j\|_1 - \phi(x_j,\epsilon_j)) \label{Sep}\\
\notag S_0^{\epsilon}(x) & = \gamma_0 (\|x\|_1 - \phi(x,\epsilon_0))
\end{align}
These smoothed sparsity penalties are shown in Figure \ref{fig:Rl1l2reg}.
\begin{figure}
\begin{center}
\epsfig{file=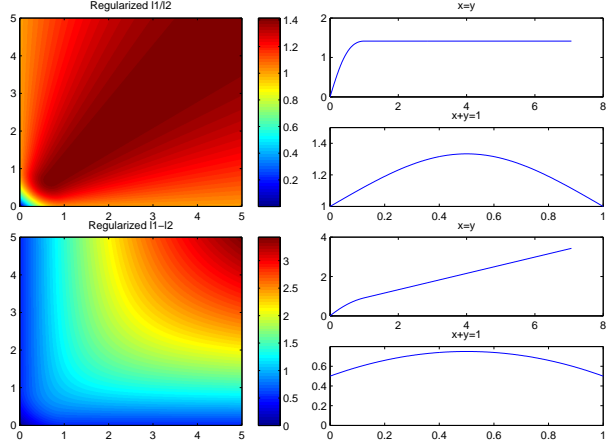,width=10cm,clip=}
\end{center}
\caption{Visualization of regularized $l_1$/$l_2$ and $l_1$ - $l_2$ penalties \label{fig:Rl1l2reg}}
\end{figure}
The regularized penalties behave more like $l_1$ near the origin and should tend to shrink $x_j$ that have small $l_2$ norms.

An alternate strategy for obtaining a differentiable objective that doesn't require smoothing the sparsity penalties is to add $M$ additional dummy variables and modify the convex constraint set.  Let $d \in \R^M$, $d \geq 0$ denote a vector of dummy variables.  Consider applying $R_j$ to vectors $\bbm x_j \\ d_j \ebm$ instead of to $x_j$.  Then if we add the constraints $\|x_j\|_1 + d_j \geq \epsilon_j$, we are assured that $R_j(x_j,d_j)$ will only be applied to nonzero vectors, even though $x_j$ is still allowed to be zero.  Moreover, by requiring that $\sum_j \frac{d_j}{\epsilon_j} \leq M - r$, we can ensure that at least $r$ of the vectors $x_j$ have one or more nonzero elements.  In particular, this prevents $x$ from being zero, so $R_0(x)$ is well defined as well.

The dummy variable strategy is our preferred approach for using the $l_1$/$l_2$ penalty.  The high variability of the regularized version near the origin creates numerical difficulties.  It either needs a lot of smoothing, which makes it behave too much like $l_1$, or its steepness near the origin makes it harder numerically to avoid getting stuck in bad local minima.  For the $l_1$ - $l_2$ penalty, the regularized approach is our preferred strategy because it is simpler and not much regularization is required.  Smoothing also makes this penalty behave more like $l_1$ near the origin, but a small shrinkage effect there may in fact be useful, especially for promoting inter group sparsity.  These two main problem formulations are summarized below as Problem 1 and Problem 2 respectively.

\noindent \textbf{Problem 1: }
\begin{align*}\label{Problem1}
\min_{x,d} F_H(x,d) & := \frac{1}{2}\|Ax-b\|^2 + \sum_{j=1}^M \gamma_j H_j(x_j,d_j) + \gamma_0 H_0(x) \\
& \notag \text{such that } x>0, d>0, \sum_{j=1}^M \frac{d_j}{\epsilon_j} \leq M - r \text{ and } \|x_j\|_1+d_j \geq \epsilon_j , \ j = 1,...,M \ .
\end{align*}

\noindent \textbf{Problem 2: }
\begin{equation*}\label{Problem2}
\min_{x \geq 0} F_S(x) := \frac{1}{2}\|Ax-b\|^2 + \sum_{j=1}^M \gamma_j S_j^{\epsilon}(x_j) + \gamma_0 S_0^{\epsilon}(x) \ . \end{equation*}


%% file: algorithm.tex
\section{Algorithm \label{algorithm}}
Both Problems 1 and 2 from Section \ref{problem} can be written abstractly as
\begin{equation}\label{eq:generalprobX} \min_{x \in X} F(x) := \frac{1}{2}\|Ax-b\|^2 + R(x) , \end{equation}
where $X$ is a convex set.  Problem 2 is already of this form with $X = \{x \in \R^N : x \geq 0 \}$.  Problem 1 is also of this form, with $X = \{ x \in \R^N, d \in \R^M : x>0, d>0, \|x_j\|_1+d_j \geq \epsilon_j, \sum_j \frac{d_j}{\epsilon_j} \leq M - r \}$.  Note that the objective function of Problem 1 can also be written as in (\ref{eq:generalprobX}) if we redefine $x_j$ as $\bbm x_j \\ d_j \ebm$ and consider an expanded vector of coefficients $x \in \R^{N+M}$ that includes the $M$ dummy variables, $d$.  The data fidelity term can still be written as $\frac{1}{2}\|Ax-b\|^2$ if columns of zeros are inserted into $A$ at the indices corresponding to the dummy variables.  In this section, we will describe algorithms and convergence analysis for solving (\ref{eq:generalprobX}) under either of two sets of assumptions.
\begin{assumption}\label{C2}
\hfill
\bi
\item $X$ is a convex set.
\item $R(x) \in \mathcal{C}^2(X,\R)$ and the eigenvalues of $\nabla^2 R(x)$ are bounded on $X$.
\item $F$ is coercive on $X$ in the sense that for any $x^0 \in X$, $\{x \in X : F(x) \leq F(x^0) \}$ is a bounded set.  In particular, $F$ is bounded below.
\ei
\end{assumption}
\begin{assumption}\label{C1concave}
\hfill
\bi
\item $R(x)$ is concave and differentiable on $X$.
\item Same assumptions on $X$ and $F$ as in Assumption \ref{C2}
\ei
\end{assumption}

Problem 1 satisfies Assumption 1 and Problem 2 satisfies Assumption 2.  We will first consider the case of Assumption 1.

Our approach for solving (\ref{eq:generalprobX}) was originally motivated by a convex splitting technique from \cite{E1998,VR} that is a semi-implicit method for solving $\frac{dx}{dt} = -\nabla F(x), \ x(0) = x^0$ when $F$ can be split into a sum of convex and concave functions $F^C(x) + F^E(x)$, both in $\mathcal{C}^2(\R^N,\R)$.  Let $\lambda^E_{\text{max}}$ be an upper bound on the eigenvalues of $\nabla^2 F^E$, and let $\lambda_{\text{min}}$ be a lower bound on the eigenvalues of $\nabla^2 F$.  Under the assumption that $\lambda^E_{\text{max}} \leq \frac{1}{2}\lambda_{\text{min}}$ it can be shown that the update defined by
\begin{equation} x^{n+1} = x^n + \Delta t (-\nabla F^C(x^{n+1}) - \nabla F^E(x^n)) \label{ugs_update} \end{equation}
doesn't increase $F$ for any time step $\Delta t > 0$.  This can be seen by using second order Taylor expansions to derive the estimate
\begin{equation} F(x^{n+1}) - F(x^n) \leq (\lambda^E_{\text{max}} - \frac{1}{2}\lambda^C_{\text{min}} - \frac{1}{\Delta t})\|x^{n+1} - x^n\|^2 . \label{ugs} \end{equation}
This convex splitting approach has been shown to be an efficient method much faster than gradient descent for solving phase-field models such as the Cahn-Hilliard equation, which has been used for example to simulate coarsening \cite{VR} and for image inpainting \cite{BEG}.

By the assumptions on $R$, we can achieve a convex concave splitting, $F = F^C + F^E$, by letting $F^C(x) = \frac{1}{2}\|Ax-b\|^2 + \|x\|_C^2$ and $F^E(x) = R(x) - \|x\|_C^2$ for an appropriately chosen positive definite matrix $C$.  We can also use the fact that $F^C(x)$ is quadratic to improve upon the estimate in (\ref{ugs}) when bounding $F(x^{n+1})-F(x^n)$ by a quadratic function of $x^{n+1}$.  Then instead of choosing a time step and updating according to (\ref{ugs_update}), we can dispense with the time step interpretation altogether and choose an update that reduces the upper bound on $F(x^{n+1})-F(x^n)$ as much as possible subject to the constraint.  This requires minimizing a strongly convex quadratic function over $X$.

\begin{proposition}\label{prop:estimate} Let Assumption \ref{C2} hold.  Also let $\lambda_r$ and $\lambda_R$ be lower and upper bounds respectively on the eigenvalues of $\nabla^2 R(x)$ for $x \in X$.  Then for $x,y \in X$ and for any matrix $C$,
\begin{equation}\label{estimate} F(y) - F(x) \leq (y-x)^T((\lambda_R-\frac{1}{2}\lambda_r)\I - C)(y-x) + (y-x)^T(\frac{1}{2}A^TA + C)(y-x) + (y-x)^T\nabla F(x) \ . \end{equation}
\end{proposition}
\begin{proof}
The estimate follows from combining several second order Taylor expansions of $F$ and $R$ with our assumptions.  First expanding $F$ about $y$ and using $h = y-x$ to simplify notation, we get that
\[ F(x) = F(y) - h^T\nabla F(y) + \frac{1}{2}h^T \nabla^2 F(y - \alpha_1h)h \]
for some $\alpha_1 \in (0,1)$.  Substituting $F$ as defined by (\ref{eq:generalprobX}), we obtain
\begin{equation}\label{estF}
F(y)-F(x) = h^T(A^TAy - A^Tb + \nabla R(y)) - \frac{1}{2}h^TA^TAh - \frac{1}{2}h^T\nabla^2 R(y - \alpha_1h)h \end{equation}
Similarly, we can compute Taylor expansions of $R$ about both $x$ and $y$.
\[ R(x) = R(y) -h^T\nabla R(y) + \frac{1}{2}h^T\nabla^2 R(y - \alpha_2h)h \ . \]
\[ R(y) = R(x) + h^T\nabla R(x) + \frac{1}{2}h^T\nabla^2 R(x + \alpha_3h)h \ . \]
Again, both $\alpha_2$ and $\alpha_3$ are in $(0,1)$.  Adding these expressions implies that
\[ h^T(\nabla R(y) - \nabla R(x)) = \frac{1}{2} h^T \nabla^2 R(y - \alpha_2h)h + \frac{1}{2}h^T \nabla^2 R(x + \alpha_3h)h \ . \]
From the assumption that the eigenvalues of $\nabla^2 R$ are bounded above by $\lambda_R$ on $X$,
\begin{equation}\label{estR} h^T(\nabla R(y) - \nabla R(x)) \leq \lambda_R \|h\|^2 \ . \end{equation}
Adding and subtracting $h^T \nabla R(x)$ and $h^T A^TA x$ to (\ref{estF}) yields
\begin{align*}
F(y) - F(x) & = h^T A^TA h + h^T(A^TAx - A^Tb + \nabla R(x)) + h^T(\nabla R(y) - \nabla R(x)) \\
& \qquad  - \frac{1}{2}h^TA^TAh - \frac{1}{2}h^T \nabla^2 R(y - \alpha_1 h)h \\
& = \frac{1}{2}h^TA^TAh + h^T \nabla F(x) +h^T(\nabla R(y) - \nabla R(x)) - \frac{1}{2}h^T \nabla^2 R(y - \alpha_1 h)h \ . \end{align*}
Using (\ref{estR}),
\[ F(y) - F(x) \leq \frac{1}{2}h^TA^TAh + h^T \nabla F(x) - \frac{1}{2}h^T \nabla^2 R(y - \alpha_1 h)h + \lambda_R\|h\|^2 \ . \]
The assumption that the eigenvalues of $\nabla^2 R(x)$ are bounded below by $\lambda_r$ on $X$ means \[ F(y) - F(x) \leq (\lambda_R - \frac{1}{2}\lambda_r)\|h\|^2 + \frac{1}{2}h^TA^TAh + h^T \nabla F(x) \ . \]
Since the estimate is unchanged by adding and subtracting $h^TCh$ for any matrix $C$,
the inequality in (\ref{estimate}) follows directly.
\end{proof}

\begin{corollary} Let $C$ be symmetric positive definite and let $\lambda_c$ denote the smallest eigenvalue of $C$.  If $\lambda_c \geq \lambda_R - \frac{1}{2}\lambda_r$, then for $x,y \in X$,
\[ F(y) - F(x) \leq (y-x)^T(\frac{1}{2}A^TA + C)(y-x) + (y-x)^T\nabla F(x) \ . \]
\end{corollary}

A natural strategy for solving (\ref{eq:generalprobX}) is then to iterate
\begin{equation}\label{QP} x^{n+1} = \arg \min_{x \in X} (x-x^n)^T(\frac{1}{2}A^TA + C_n)(x-x^n) + (x-x^n)^T\nabla F(x^n) \end{equation} for $C_n$ chosen to guarantee a sufficient decrease in $F$.  The method obtained by iterating (\ref{QP}) can be viewed as an instance of scaled gradient projection \cite{BT,B,BZZ} where the orthogonal projection of $x^n - (A^TA+2C_n)^{-1}\nabla F(x^n)$ onto $X$ is computed in the norm $\|\cdot\|_{A^TA+2C_n}$.  The approach of decreasing $F$ by minimizing an upper bound coming from an estimate like (\ref{estimate}) can be interpreted as an optimization transfer strategy of defining and minimizing a surrogate function \cite{LHY}, which is done for related applications in \cite{H2002,LS}.  It can also be interpreted as a special case of difference of convex programming \cite{TA}.

Choosing $C_n$ in such a way that guarantees $(x^{n+1}-x^n)^T((\lambda_R-\frac{1}{2}\lambda_r)\I - C_n)(x^{n+1}-x^n) \leq 0$ may be numerically inefficient, and it also isn't strictly necessary for the algorithm to converge.  To simplify the description of the algorithm, suppose $C_n = c_n C$ for some scalar $c_n > 0$ and symmetric positive definite $C$.  Then as $c_n$ gets larger, the method becomes more like explicit gradient projection with small time steps.  This can be slow to converge as well as more prone to converging to bad local minima.  However, the method still converges as long as each $c_n$ is chosen so that the $x^{n+1}$ update decreases $F$ sufficiently.  Therefore we want to dynamically choose $c_n \geq 0$ to be as small as possible such that the $x^{n+1}$ update given by (\ref{QP}) decreases $F$ by a sufficient amount, namely
\[ F(x^{n+1}) - F(x^n) \leq \sigma \left[ (x^{n+1}-x^n)^T(\frac{1}{2}A^TA + C_n)(x^{n+1}-x^n) + (x^{n+1}-x^n)^T\nabla F(x^n) \right] \]
for some $\sigma \in (0,1]$.  Additionally, we want to ensure that the modulus of strong convexity of the quadratic objective in (\ref{QP}) is large enough by requiring the smallest eigenvalue of $\frac{1}{2}A^TA + C_n$ to be greater than or equal to some $\rho >0$.  The following is an algorithm for solving (\ref{eq:generalprobX}) and a dymamic update scheme for $C_n = c_n C$ that is similar to Armijo line search but designed to reduce the number of times that the solution to the quadratic problem has to be rejected for not decreasing $F$ sufficiently.
\alg{
\label{SGPalg_dynamic}
A Scaled Gradient Projection Method for Solving (\ref{eq:generalprobX}) Under Assumption \ref{C2} \\ \ \\
\noindent Define $x^0 \in X$, $c_0 > 0$, $\sigma \in (0,1]$, $\epsilon > 0$, $\rho > 0$, $\xi_1 > 1, \xi_2 > 1$ and set $n = 0$. \\

\texttt{while} \ $n=0$ \ or \  $\|x^n - x^{n-1}\|_{\infty} > \epsilon$
\[ y = \arg \min_{x \in X} \ (x-x^n)^T(\frac{1}{2}A^TA + c_n C)(x-x^n) + (x-x^n)^T \nabla F(x^n) \qquad \qquad \]

\qquad \texttt{if} \ $F(y)-F(x^n) > \sigma \left[ (y-x^n)^T(\frac{1}{2}A^TA + c_n C)(y-x^n) + (y-x^n)^T \nabla F(x^n) \right]$ \\

\qquad \qquad $c_n = \xi_2 c_n$ \\

\qquad \texttt{else} \\

\qquad \qquad $x^{n+1} = y$

\qquad \qquad $c_{n+1} = \begin{cases} \frac{c_n}{\xi_1} & \text{if smallest eigenvalue of }  \frac{c_n}{\xi_1}C + \frac{1}{2}A^TA \ \text{is greater than }  \rho \\
c_n & \qquad \text{otherwise} \ . \end{cases}$


\qquad \qquad $n = n + 1$ \\

\qquad \texttt{end if} \\

\texttt{end while}
}

It isn't necessary to impose an upper bound on $c_n$ in Algorithm \ref{SGPalg_dynamic} even though we want it to be bounded.  The reason for this is because once $c_n \geq \lambda_R - \frac{1}{2}\lambda_r$, $F$ will be sufficiently decreased for any choice of $\sigma \in (0,1]$, so $c_n$ is effectively bounded by $\xi_2(\lambda_R - \frac{1}{2}\lambda_r)$.

Under Assumption \ref{C1concave} it is much more straightforward to derive an estimate analogous to Proposition \ref{prop:estimate}.  Concavity of $R(x)$ immediately implies
\[R(y) \leq R(x) + (y-x)^T \nabla R(x) \ . \]
Adding to this the expression
\[ \frac{1}{2}\|Ay-b\|^2 = \frac{1}{2}\|Ax-b\|^2 + (y-x)^T (A^TAx - A^Tb) + \frac{1}{2}(y-x)^T A^TA(y-x) \]
yields
\begin{equation}\label{eq:estimate_dc}
F(y) - F(x) \leq (y-x)^T \frac{1}{2}A^TA (y-x) + (y-x)^T \nabla F(x)
\end{equation}
for $x,y \in X$.  Moreover, the estimate still holds if we add $(y-x)^T C (y-x)$ to the right hand side for any positive semi-definite matrix $C$.  We are again led to iterating (\ref{QP}) to decrease $F$, and in this case $C_n$ need only be included to ensure that $A^TA+2C_n$ is positive definite.
We can let $C_n = C$ since the dependence on $n$ is no longer necessary.  We can choose any $C$ such that the smallest eigenvalue of $C + \frac{1}{2}A^TA$ is greater than $\rho > 0$, but it is still preferable to choose $C$ as small as is numerically practical.
\alg{
\label{SGPalg}
A Scaled Gradient Projection Method for Solving (\ref{eq:generalprobX}) Under Assumption \ref{C1concave} \\ \ \\
\noindent Define $x^0 \in X$, $C$ symmetric positive definite and $\epsilon > 0$. \\

\texttt{while} \ $n=0$ \ or \  $\|x^n - x^{n-1}\|_{\infty} > \epsilon$
\begin{equation}\label{QPc} x^{n+1} = \arg \min_{x \in X} \ (x-x^n)^T(\frac{1}{2}A^TA + C)(x-x^n) + (x-x^n)^T \nabla F(x^n) \qquad \qquad \end{equation}

\qquad \qquad $n = n + 1$ \\

\texttt{end while}
}
Since the objective in (\ref{QPc}) is zero at $x = x^n$, the minimum value is less than or equal to zero, and so $F(x^{n+1}) \leq F(x^n)$ by (\ref{eq:estimate_dc}).

Algorithm \ref{SGPalg} is also equivalent to iterating
\[ x^{n+1} = \arg \min_{x \in X} \frac{1}{2}\|Ax-b\|^2 + \|x\|_C^2 + x^T (\nabla R(x^n) - 2Cx^n) \ , \]
which can be seen as an application of the simplified difference of convex algorithm from \cite{TA} to $F(x) = (\frac{1}{2}\|Ax-b\|^2 + \|x\|_C^2) - (-R(x) + \|x\|_C^2)$.  The DC method in \cite{TA} is more general and doesn't require the convex and concave functions to be differentiable.

With many connections to classical algorithms, existing convergence results can be applied to argue that limit points of the iterates $\{x^n\}$ of Algorithms \ref{SGPalg} and \ref{SGPalg_dynamic} are stationary points of (\ref{eq:generalprobX}).  We still choose to include a convergence analysis for clarity because our assumptions allow us to give a simple and intuitive argument.  The following analysis is for Algorithm \ref{SGPalg_dynamic} under Assumption \ref{C2}.  However, if we replace $C_n$ with $C$ and $\sigma$ with $1$, then it applies equally well to Algorithm \ref{SGPalg} under Assumption \ref{C1concave}.  We proceed by showing that the sequence $\{x^n\}$ is bounded, $\|x^{n+1}-x^n\| \rightarrow 0$ and limit points of $\{x^n\}$ are stationary points of (\ref{eq:generalprobX}) satisfying the necessary local optimality condition (\ref{eq:genopt}).

\begin{lemma}\label{bounded}
The sequence of iterates $\{x^n\}$ generated by Algorithm \ref{SGPalg_dynamic} is bounded.
\end{lemma}
\begin{proof}
Since $F(x^n)$ is non-increasing, $x^n \in \{ x \in X : F(x) \leq F(x^0) \}$, which is a bounded set by assumption.
\end{proof}

\begin{lemma}\label{slowing}
Let $\{x^n\}$ be the sequence of iterates generated by Algorithm \ref{SGPalg_dynamic}.  Then $\|x^{n+1}-x^n\| \rightarrow 0$.
\end{lemma}
\begin{proof}
Since $\{F(x^n)\}$ is bounded below and non-increasing, it converges.  By construction, $x^{n+1}$ satisfies
\[ - \left[ (x^{n+1}-x^n)^T(\frac{1}{2}A^TA + C_n)(x^{n+1}-x^n) + (x^{n+1}-x^n)^T\nabla F(x^n) \right] \leq \frac{1}{\sigma}(F(x^n) - F(x^{n+1})) \ . \]
By the optimality condition for (\ref{QP}),
\[ (y-x^{n+1})^T\left( (A^TA + 2C_n)(x^{n+1}-x^n) + \nabla F(x^n) \right) \geq 0 \qquad \forall y \in X \ . \]
In particular, we can take $y = x^n$, which implies
\[ (x^{n+1}-x^n)^T(A^TA + 2C_n)(x^{n+1}-x^n) \leq -(x^{n+1}-x^n)^T \nabla F(x^n) \ . \]
Thus
\[ (x^{n+1}-x^n)^T(\frac{1}{2}A^TA + C_n)(x^{n+1}-x^n) \leq \frac{1}{\sigma}(F(x^n) - F(x^{n+1})) \ . \]
Since the eigenvalues of $\frac{1}{2}A^TA + C_n$ are bounded below by $\rho > 0$, we have that
\[ \rho \|x^{n+1} - x^n\|^2 \leq \frac{1}{\sigma}(F(x^n) - F(x^{n+1})) \ . \]
The result follows from noting that
\[ \lim_{n \rightarrow \infty} \|x^{n+1}-x^n\|^2 \leq \lim_{n \rightarrow \infty} \frac{1}{\sigma \rho}(F(x^n) - F(x^{n+1})) , \]
which equals $0$ since $\{F(x^n)\}$ converges.
\end{proof}
\begin{proposition}\label{limitpoint}
Any limit point $x^*$ of the sequence of iterates $\{x^n\}$ generated by Algorithm \ref{SGPalg_dynamic} satisfies $(y-x^*)^T \nabla F(x^*) \geq 0$ for all $y \in X$, which means $x^*$ is a stationary point of (\ref{eq:generalprobX}).
\end{proposition}
\begin{proof}
Let $x^*$ be a limit point of $\{x^n\}$.  Since $\{x^n\}$ is bounded, such a point exists.  Let $\{x^{n_k}\}$ be a subsequence that converges to $x^*$.  Since $\|x^{n+1} - x^n\| \rightarrow 0$, we also have that $x^{n_k + 1} \rightarrow x^*$.  Recalling the optimality condition for (\ref{QP}),
\begin{align*} 0 & \leq (y-x^{n_k+1})^T\left( (A^TA + 2C_{n_k})(x^{n_k+1}-x^{n_k}) + \nabla F(x^{n_k}) \right) \leq \\
& \|y - x^{n_k + 1}\|\|A^TA + 2C_{n_k}\|\|x^{n_k + 1} - x^{n_k}\| + (y-x^{n_k + 1})^T \nabla F(x^{n_k}) \qquad \forall y \in X \ . \end{align*}
Following \cite{BT}, proceed by taking the limit along the subsequence as $n_k \rightarrow \infty$.
\[ \|y - x^{n_k + 1}\|\|x^{n_k + 1} - x^{n_k}\|\|A^TA + 2C_{n_k}\| \rightarrow 0 \]
since $\|x^{n_k+1} - x^{n_k}\| \rightarrow 0$ and $\|A^TA + 2C_{n_k}\|$ is bounded.  By continuity of $\nabla F$ we get that
\[ (y-x^*)^T \nabla F(x^*) \geq 0 \qquad \forall y \in X \ . \]
\end{proof}

Each iteration requires minimizing a strongly convex quadratic function over the set $X$ as defined in (\ref{QP}).  Many methods can be used to solve this, and we want to choose one that is as robust as possible to poor conditioning of $\frac{1}{2}A^TA + C_n$.  For example, gradient projection works theoretically and even converges at a linear rate, but it can still be impractically slow.  A better choice here is to use the alternating direction method of multipliers (ADMM) \cite{GM,GlM}, which alternately solves a linear system involving $\frac{1}{2}A^TA + C_n$ and projects onto the constraint set.  Applied to Problem 2, this is essentially the same as the application of split Bregman \cite{GO} to solve a NNLS model for hyperspectral demixing in \cite{SGO}.  We consider separately the application of ADMM to Problems 1 and 2.  The application to Problem 2 is simpler.

For Problem 2, (\ref{QP}) can be written as
\[ x^{n+1} = \arg \min_{x \geq 0} (x-x^n)^T(\frac{1}{2}A^TA + C_n)(x-x^n) + (x-x^n)^T \nabla F_S(x^n) \ . \]
To apply ADMM, we can first reformulate the problem as
\begin{equation}\label{QP2}
\min_{u,v} g_{\geq 0}(v) + (u-x^n)^T(\frac{1}{2}A^TA + C_n)(u-x^n) + (u-x^n)^T \nabla F_S(x^n) \quad \text{such that} \quad u=v \ , \end{equation}
where $g$ is an indicator function for the constraint defined by $g_{\geq 0}(v) = \begin{cases}0 & v \geq 0 \\ \infty & \text{otherwise} \end{cases}$.

Introduce a Lagrange multiplier $p$ and define a Lagrangian
\begin{equation}\label{L}
L(u,v,p) = g_{\geq 0}(v) + (u-x^n)^T(\frac{1}{2}A^TA + C_n)(u-x^n) + (u-x^n)^T \nabla F_S(x^n) + p^T(u-v) \end{equation}
and augmented Lagrangian
\[ L_{\delta}(u,v,p) = L(u,v,p) + \frac{\delta}{2}\|u-v\|^2 \ , \]
where $\delta > 0$.  ADMM finds a saddle point \[ L(u^*,v^*,p) \leq L(u^*,v^*,p^*) \leq L(u,v,p^*) \qquad \forall u,v,p \] by alternately minimizing $L_{\delta}$ with respect to $u$, minimizing with respect to $v$ and updating the dual variable $p$.  Having found a saddle point of $L$, $(u^*,v^*)$ will be a solution to (\ref{QP2}) and we can take $v^*$ to be the solution to (\ref{QP}). The explicit ADMM iterations are described in the following algorithm.
\alg{
\label{QP2_ADMM}
ADMM for solving convex subproblem for Problem 2 \\ \ \\
\noindent Define $\delta > 0$, $v^0$ and $p^0$ arbitrarily and let $k = 0$. \\

\texttt{while} \text{not converged}
\begin{align*}
u^{k+1} & = x^n + (A^TA + 2C_n + \delta \I)^{-1} \left( \delta(v^k - x^n) - p^k - \nabla F_S(x^n) \right) \\
v^{k+1} & = \Pi_{\geq 0} \left( u^{k+1} + \frac{p^k}{\delta} \right) \\
p^{k+1} & = p^k + \delta (u^{k+1} - v^{k+1})
\end{align*}

\qquad \qquad \ \ \  k = k + 1 \\

\texttt{end while}
}

Here $\Pi_{\geq 0}$ denotes the orthogonal projection onto the non-negative orthant.  For this application of ADMM to be practical, $(A^TA + 2C_n + \delta \I)^{-1}$ should not be too expensive to apply, and $\delta$ should be well chosen.

Since (\ref{QP}) is a standard quadratic program, a huge variety of other methods could also be applied.  Variants of Newton's method on a bound constrained KKT system might work well here, especially if we find we need to solve the subproblem to very high accuracy.




For Problem 2, (\ref{QP}) can be written as
\begin{align*} (x^{n+1},d^{n+1})&  = \arg \min_{x,d} (x-x^n)^T(\frac{1}{2}A^TA + C^x_n)(x-x^n) + (d-d^n)^T C^d_n (d-d^n) + \\
& (x-x^n)^T \nabla_x F_H(x^n,d^n) + (d-d^n)^T \nabla_d F_H(x^n,d^n) \ . \end{align*}
Here, $\nabla_x$ and $\nabla_d$ represent the gradients with respect to $x$ and $d$ respectively.  The matrix $C_n$ is assumed to be of the form $C_n = \bbm C^x_n & 0 \\ 0 & C^d_n \ebm$, with $C^d_n$ a diagonal matrix.
It is helpful to represent the constraints in terms of convex sets defined by
\[X_{\epsilon_j} = \left\{ \bbm x_j \\ d_j \ebm \in \R^{m_j+1} : \|x_j\|_1 + d_j \geq \epsilon_j, \qquad x_j \geq 0, \qquad d_j \geq 0 \right\} \ j = 1,...,M \ , \]
\[X_{\beta} = \left\{ d \in \R^M : \sum_{j=1}^M \frac{d_j}{\beta_j} \leq M - r , \qquad d_j \geq 0\right\} \ , \]
and indicator functions $g_{X_{\epsilon_j}}$ and $g_{X_{\beta}}$ for these sets.

Let $u$ and $w$ represent $x$ and $d$.  Then by adding splitting variables $v_x = u$ and $v_d = w$ we can reformulate the problem as
\begin{align*} \min_{u,w,v_x,v_d} & \sum_j g_{X_{\epsilon_j}}({v_x}_j,{v_d}_j) + g_{X_{\beta}}(w) + (u-x^n)^T(\frac{1}{2}A^TA + C^x_n)(u-x^n) + (w-d^n)^T C^d_n (w-d^n) + \\
& (x-x^n)^T \nabla_x F_H(x^n,d^n) + (w-d^n)^T \nabla_d F_H(x^n,d^n) \qquad \text{s.t.} \qquad v_x = u, v_d = w \ . \end{align*}
Adding Lagrange multipliers $p_x$ and $p_d$ for the linear constraints, we can define the augmented Lagrangian
\begin{align*} L_{\delta}(u,w,v_x,v_d,p_x,p_d) & = \sum_j g_{X_{\epsilon_j}}({v_x}_j,{v_d}_j) + g_{X_{\beta}}(w) + (u-x^n)^T(\frac{1}{2}A^TA + C^x_n)(u-x^n) + \\
& (w-d^n)^T C^d_n (w-d^n) + (x-x^n)^T \nabla_x F_H(x^n,d^n) + (w-d^n)^T \nabla_d F_H(x^n,d^n) + \\
& p_x^T(u - v_x) + p_d^T(w - v_d) + \frac{\delta}{2}\|u - v_x\|^2 + \frac{\delta}{2}\|w - v_d\|^2 \ . \end{align*}
Each ADMM iteration alternately minimizes $L_{\delta}$ first with respect to $(u,w)$ and then with respect to $(v_x,v_d)$ before updating the dual variables $(p_x,p_d)$.  The explicit iterations are described in the following algorithm.
\alg{
\label{QP1_ADMM}
ADMM for solving convex subproblem for Problem 1 \\ \ \\
\noindent Define $\delta > 0$, $v_x^0$, $v_d^0$, $p_x^0$ and $p_d^0$ arbitrarily and let $k = 0$. \\
\noindent Define the weights $\beta$ in the projection $\Pi_{X_{\beta}}$ by $\beta_j = (\epsilon_j \sqrt{(2C^d_n + \delta \I)_{j,j}})^{-1} \ \ j = 1,...,M$. \\

\texttt{while} \text{not converged}
\begin{align*}
u^{k+1} & = x^n + (A^TA + 2C^x_n + \delta \I)^{-1} \left( \delta(v_x^k - x^n) - p_x^k - \nabla_x F_H(x^n,d^n) \right) \\
w^{k+1} & = (2C^d_n + \delta \I)^{-\frac{1}{2}} \Pi_{X_{\beta}}\left( (2C^d_n + \delta \I)^{-\frac{1}{2}}(\delta v_d^k - p_d^k - \nabla_d F_H(x^n,d^n) + 2 C^d_n)   \right) \\
{\bbm {v_x}_j \\ {v_d}_j \ebm}^{k+1} & = \Pi_{X_{\epsilon_j}} \left( \bbm u_j^{k+1} + \frac{{p_x}_j^k}{\delta} \\ w_j^{k+1} + \frac{p_{d_j}^k}{\delta} \ebm \right) \ j = 1,...,M \\
p_x^{k+1} & = p_x^k + \delta (u^{k+1} - v_x^{k+1}) \\
p_d^{k+1} & = p_d^k + \delta (w^{k+1} - v_d^{k+1})
\end{align*}

\qquad \qquad \ \ k = k + 1 \\

\texttt{end while}
}
We stop iterating and let $x^{n+1} = v_x$ and $d^{n+1} = v_d$ once the relative errors of the primal and dual variables are sufficiently small.  The projections $\Pi_{X_{\beta}}$ and $\Pi_{X_{\epsilon_j}}$ can be efficiently computed by combining projections onto the non-negative orthant and
projections onto the appropriate simplices.  These can in principle be computed in linear time \cite{Brucker}, although we use a method that is simpler to implement and is still only $O(n \log n)$ in the dimension of the vector being projected.



%% file: applications.tex
\section{Applications \label{applications}}
In this section we introduce four specific applications related to DOAS analysis and hyperspectral demixing.  We show how to model these problems in the form of (\ref{eq:generalprobX}) so that the algorithms from Section \ref{algorithm} can be applied.

\subsection{DOAS Analysis}
The goal of DOAS is to estimate the concentrations of gases in a mixture by measuring over a range of wavelengths the reduction in the intensity of light shined through it.  A thorough summary of the procedure and analysis can be found in \cite{PS}.

Beer's law can be used to estimate the attenuation of light intensity due to absorption.  Assuming the average gas concentration $c$ is not too large, Beer's law relates the transmitted intensity $I(\lambda)$ to the initial intensity $I_0(\lambda)$ by
\begin{equation}\label{beer} I(\lambda) = I_0(\lambda) \exp^{-\sigma(\lambda) c L} , \end{equation}
where $\lambda$ is wavelength, $\sigma(\lambda)$ is the characteristic absorption spectra for the absorbing gas and $L$ is the light path length.

If the density of the absorbing gas is not constant, we should instead integrate over the light path, replacing $\exp^{-\sigma(\lambda) c L}$ by $\exp^{-\sigma(\lambda) \int_0^L c(l) dl}$.  For simplicity, we will assume the concentration is approximately constant.  We will also denote the product of concentration and path length, $cL$, by $a$.

When multiple absorbing gases are present, $a \sigma(\lambda)$ can be replaced by a linear combination of the characteristic absorption spectra of the gases, and Beer's law can be written as
\[ I(\lambda) = I_0(\lambda) \exp^{-\sum_j a_j \sigma_j (\lambda) } . \]

Additionally taking into account the reduction of light intensity due to scattering, combined into a single term $\epsilon(\lambda)$, Beer's law becomes
\[ I(\lambda) = I_0(\lambda) \exp^{-\sum_j a_j \sigma_j (\lambda)  - \epsilon(\lambda)} . \]

The key idea behind DOAS is that it is not necessary to explicitly model effects such as scattering, as long as they vary smoothly enough with wavelength to be removed by high pass filtering that loosely speaking removes the broad structures and keeps the narrow structures.  We will assume that $\epsilon(\lambda)$ is smooth.  Additionally, we can assume that $I_0(\lambda)$, if not known, is also smooth.  The absorption spectra $\sigma_j(\lambda)$ can be considered to be a sum of a broad part (smooth) and a narrow part, $\sigma_j = \sigma_j^{\text{broad}}$ + $\sigma_j^{\text{narrow}}$.  Since $\sigma_j^{\text{narrow}}$ represents the only narrow structure in the entire model, the main idea is to isolate it by taking the log of the intensity and applying high pass filtering or any other procedure, such as polynomial fitting, that subtracts a smooth background from the data.  The given reference spectra should already have had their broad parts subtracted, but it may not have been done consistently, so we will combine $\sigma_j^{\text{broad}}$ and $\epsilon(\lambda)$ into a single term $B(\lambda)$.  We will also denote the given reference spectra by $y_j$, which again are already assumed to be approximately high pass filtered versions of the true absorption spectra $\sigma_j$.  With these notational changes, Beer's law becomes
\begin{equation}\label{beer_broad} I(\lambda) = I_0(\lambda) \exp^{-\sum_j a_j y_j (\lambda)  - B(\lambda)} . \end{equation}

In practice, measurement errors must also be modeled.  We therefore consider multiplying the right hand side of (\ref{beer_broad}) by $s(\lambda)$, representing wavelength dependent sensitivity.  Assuming that $s(\lambda) \approx 1$ and varies smoothly with $\lambda$, we can absorb it into $B(\lambda)$.  Measurements may also be corrupted by convolution with an instrument function $h(\lambda)$, but for simplicity we will assume this effect is negligible and not include convolution with $h$ in the model.  Let $J(\lambda) = -\ln(I(\lambda))$.  This is what we will consider to be the given data.  By taking the log, the previous model simplifies to
\[ J(\lambda) = -\ln(I_0(\lambda)) + \sum_j a_j y_j(\lambda) + B(\lambda) + \eta(\lambda) , \]
where $\eta(\lambda)$ represents the log of multiplicative noise, which we will model as being approximately white Gaussian noise.

Since $I_0(\lambda)$ is assumed to be smooth, it can also be absorbed into the $B(\lambda)$ component, yielding the data model
\begin{equation}\label{DOAS_data_model1} J(\lambda) = \sum_j a_j y_j(\lambda) + B(\lambda) + \eta(\lambda) . \end{equation}

\subsubsection{DOAS Analysis with Wavelength Misalignment \label{sec:basic_DOAS}}
A challenging complication in practice is wavelength misalignment,
\emph{i.e.}, the nominal wavelengths in the measurement $J(\lambda)$
may not correspond exactly to those in the basis $y_j(\lambda)$.  We must allow for small, often approximately linear deformations $v_j(\lambda)$ so that $y_j(\lambda + v_j(\lambda))$ are all aligned with the data $J(\lambda)$.  Taking into account wavelength misalignment, the data model becomes
\begin{equation}\label{DOAS_data_model2} J(\lambda) = \sum_j a_j y_j(\lambda + v_j(\lambda)) + B(\lambda) + \eta(\lambda) . \end{equation}

To first focus on the alignment aspect of this problem, assume $B(\lambda)$ is negligible, having somehow been consistently removed from the data and references by high pass filtering or polynomial subtraction.  Then given the data $J(\lambda)$ and reference spectra
$\{y_j(\lambda)\}$, we want to estimate the fitting coefficients
$\{a_j\}$ and the deformations $\{v_j(\lambda)\}$ from the linear model,
\begin{equation}\label{eq:model_with_alignment}
J(\lambda) =\sum_{j=1}^M a_j y_j\big(\lambda+v_j(\lambda)\big)+\eta(\lambda)\ ,
\end{equation}
where $M$ is the total number of gases to be considered.

Inspired by the idea of using a set of modified bases for image
deconvolution \cite{louBS11}, we construct a dictionary by deforming
each $y_j$ with a set of possible deformations. Specifically, since the deformations can be well
approximated by linear functions, \emph{i.e.},
$v_j(\lambda)=p_j\lambda+q_j$, we enumerate all the possible
deformations by choosing $p_j, q_j$ from two pre-determined sets
$\{P_1,\cdots, P_K\}$, $\{Q_1,\cdots, Q_L\}$. Let $A_j$ be a matrix
whose columns are deformations of the $j$th reference $y_j(\lambda)$, \emph{i.e.}, $ y_j(\lambda+P_k\lambda+Q_l)$ for
$k=1,\cdots,K$ and $l=1,\cdots, L$. Then we can rewrite the model
(\ref{eq:model_with_alignment}) in terms of a matrix-vector form,
\begin{equation}\label{eq:model_dict}
J = [A_1, \cdots, A_M]\left[\begin{array}{c}x_1 \\ \vdots\\
x_M\end{array}\right] + \eta \ ,
\end{equation}
where $x_j \in \R^{KL}$ and $J \in \R^W$.  

We propose the following minimization model,
\begin{equation}
\left.\begin{array}{l}  \arg \min_{x_j} \frac{1}{2} \|J - [A_1, \cdots, A_M]\left[\begin{array}{c}x_1 \\ \vdots \\
x_M\end{array}\right]\|^2\ ,\\
\mbox{s.t.} \quad x_j\geqslant 0, \ \|x_j\|_0 \leqslant 1 \qquad
j=1,\cdots,M \ . \end{array}\right.\label{eq:opt_constraint}
\end{equation}
The second constraint in (\ref{eq:opt_constraint}) is to enforce
each $x_j$ to have at most one non-zero element.  Having $\|x_j\|_0 = 1$
indicates the existence of the gas with a spectrum $y_j$  Its
non-zero index corresponds to the selected deformation and its
magnitude corresponds to the concentration of the gas. This $l_0$
constraint makes the problem NP-hard.  A direct approach is the
penalty decomposition method proposed in \cite{luZ12}, which we will compare to in Section \ref{numerics}.  Our approach is to replace the $l_0$ constraint on each group with intra sparsity penalties defined by $H_j$ in (\ref{H}) or $S_j^{\epsilon}$ in (\ref{Sep}), putting the problem in the form of Problem 1 or Problem 2.  The intra sparsity parameters $\gamma_j$ should be chosen large enough to enforce 1-sparsity within groups, and in the absence of any inter group sparsity assumptions we can set $\gamma_0 = 0$.

\subsubsection{DOAS with Background Model}
To incorporate the background term from (\ref{DOAS_data_model2}), we will add $B \in \R^W$ as an additional unknown and also add a quadratic penalty $\frac{\alpha}{2} \|Q B\|^2$ to penalize a lack of smoothness of $B$.  This leads to the model
\[ \min_{x \in X, B} \frac{1}{2}\|Ax + B - J\|^2 + \frac{\alpha}{2}\|Q B\|^2 + R(x) \ , \]
where $R$ includes our choice of intra sparsity penalties on $x$.  This can be rewritten as
\begin{equation}\label{eq:opt_background}\min_{x \in X, B} \frac{1}{2}\left\| \bbm A & \I \\ 0 & \sqrt{\alpha}Q \ebm \bbm x \\ B \ebm - \bbm J \\ 0 \ebm \right\|^2 + R(x) \ . \end{equation}  This has the general form of (\ref{eq:generalprobX}) with the two by two block matrix interpreted as $A$ and $\bbm J \\ 0 \ebm$ interpreted as $b$.  Moreover, we can concatenate $B$ and the $M$ groups $x_j$ by considering $B$ to be group $x_{M+1}$ and setting $\gamma_{M+1} = 0$ so that no sparsity penalty acts on the background component.  In this way, we see that the algorithms presented in Section \ref{algorithm} can be directly applied to (\ref{eq:opt_background}).

It remains to define the matrix $Q$ used in the penalty to enforce smoothness of the estimated background.  A possible strategy is to work with the discrete Fourier transform or discrete cosine transform of $B$ and penalize high frequency coefficients.  Although $B$ should be smooth, it is unlikely to satisfy Neumann or periodic boundary conditions, so based on an idea in \cite{SR}, we will work with $B$ minus the linear function that interpolates its endpoints.  Let $L \in \R^{W \times W}$ be the matrix representation of the linear operator that takes the difference of $B$ and its linear interpolant.  Since $LB$ satisfies zero boundary conditions and its odd periodic extension should be smooth, its discrete sine transform (DST) coefficients should rapidly decay.  So we can penalize the high frequency DST coefficients of $LB$ to encourage smoothness of $B$.  Let $\Gamma$ denote the DST and let $W_B$ be a diagonal matrix of positive weights that are larger for higher frequencies.  An effective choice is $\diag(W_B)_i = i^2$, since the index $i = 0,..,W-1$ is proportional to frequency.  We then define $Q = W_B \Gamma L$ in (\ref{eq:opt_background}) and can adjust the strength of this smoothing penalty by changing the single parameter $\alpha > 0$.  Figure \ref{Q} shows the weights $W_B$ and the result $LB$ of subtracting from $B$ the line interpolating its endpoints.
\begin{figure}
\begin{center}
\begin{tabular}{cc}
$W_B$ & $LB$ \\
\includegraphics[width=0.3\textwidth]{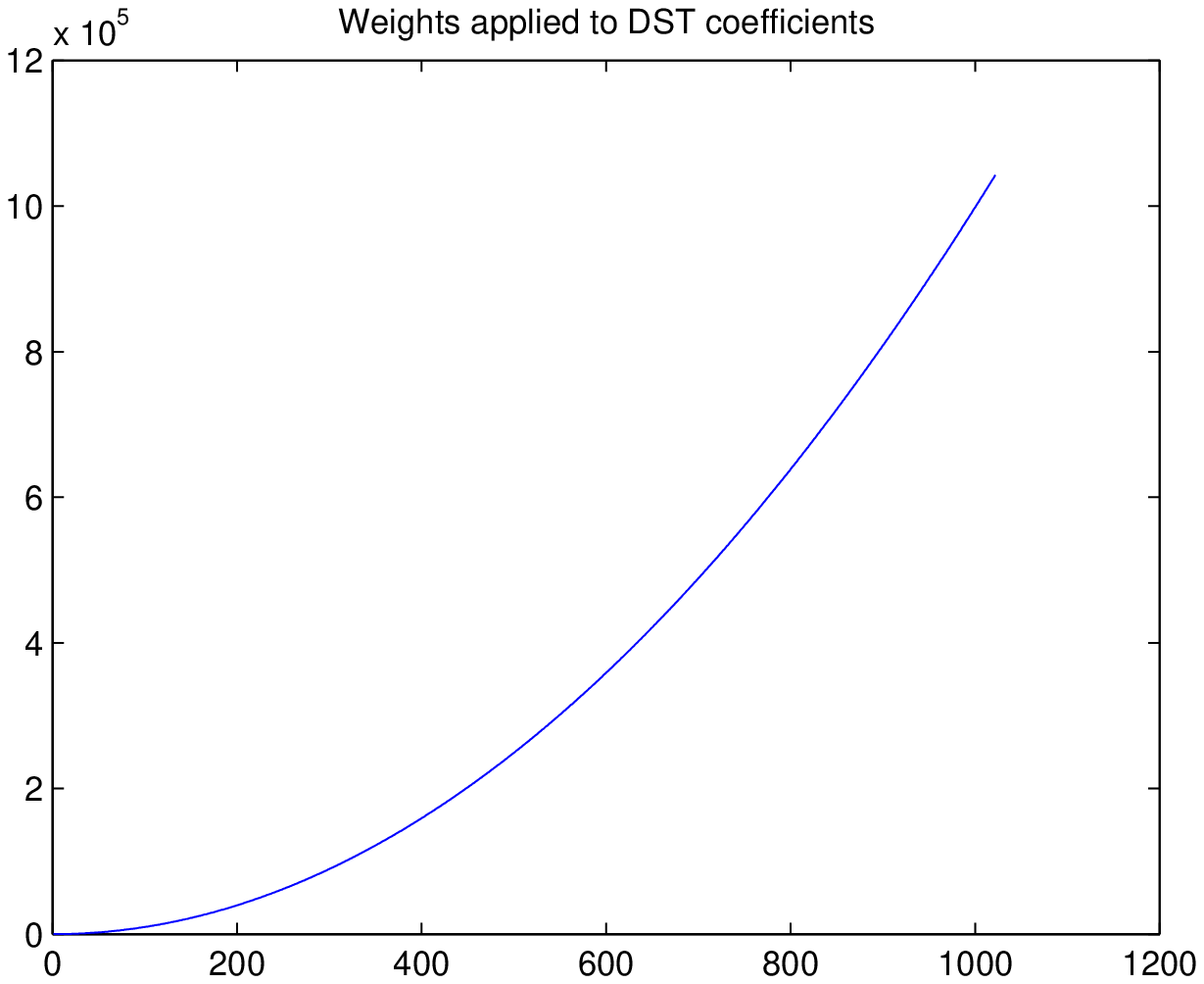}&
\includegraphics[width=0.3\textwidth]{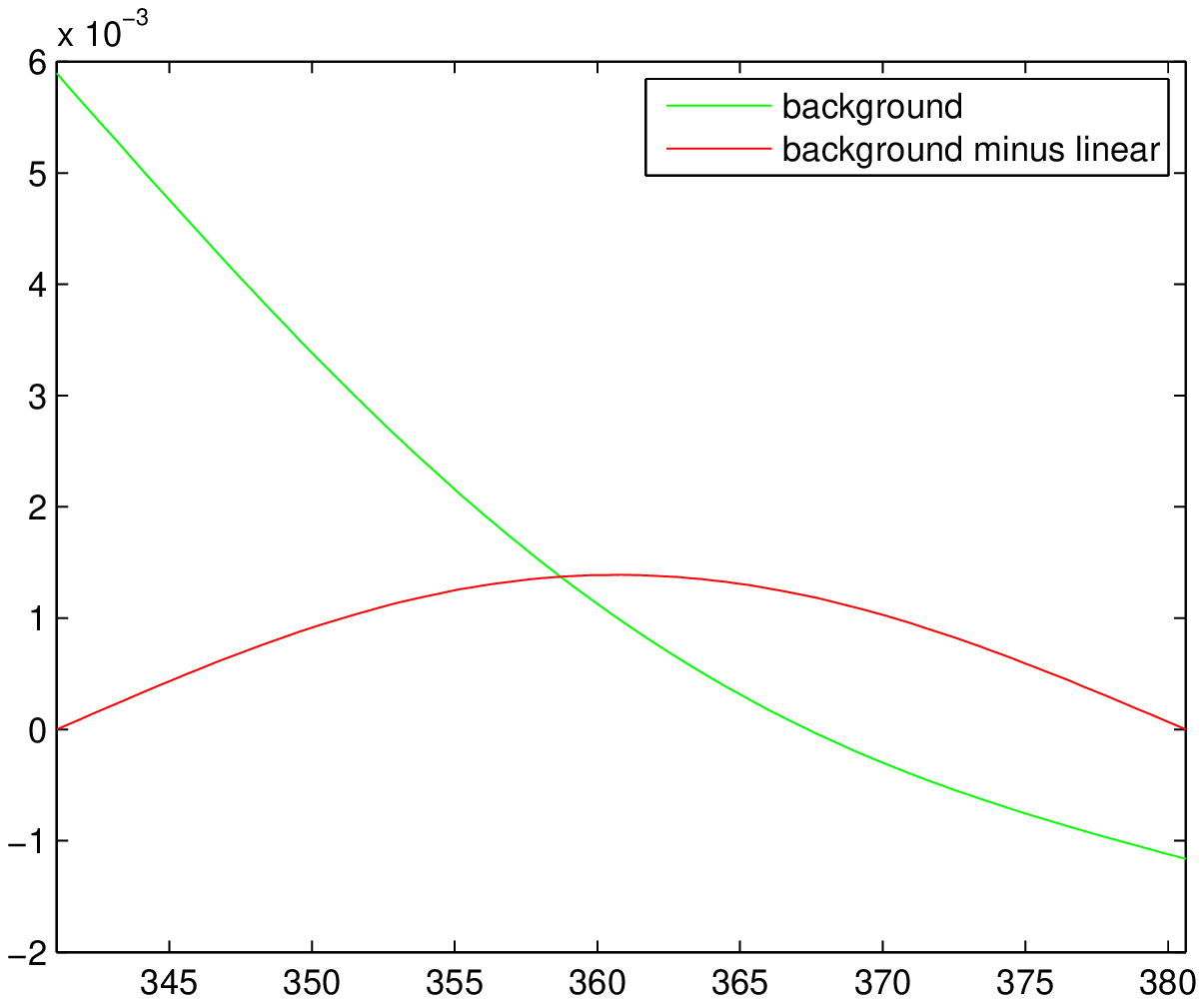}\\
\end{tabular}
\caption{Functions used to define background penalty} \label{Q}
\end{center}
\end{figure}

\subsection{Hyperspectral Image Analysis}
Hyperspectral images record high resolution spectral information at each pixel of an image.  This large amount of spectral data makes it possible to identify materials based on their spectral signatures.  A hyperspectral image can be represented as a matrix $Y \in \R^{W \times P}$, where $P$ is the number of pixels and $W$ is the number of spectral bands.

Due to low spatial resolution or finely mixed materials, each pixel can contain multiple different materials.  The spectral data measured at each pixel, according to a linear mixing model, is assumed to be a non-negative linear combination of spectral signatures of pure materials, which are called endmembers.  The list of known endmembers can be represented as the columns of a matrix $A \in \R^{W \times N}$.

The goal of hyperspectral demixing is to determine the abundances of different materials at each pixel.  Given $Y$, and if $A$ is also known, the goal is then to determine an abundance matrix $S \in \R^{N \times P}$ with $S_{i,j} \geq 0$.  Each row of $S$ is interpretable as an image that shows the abundance of one particular material at every pixel.  Mixtures are often assumed to involve only very few of the possible materials, so the columns of $S$ are often additionally assumed to be sparse.

\subsubsection{Sparse Hyperspectral Demixing \label{sec:basic_demix}}
A simple but effective approach for hyperspectral demixing is NNLS, which here is to solve
\[ \min_{S \geq 0} \|Y - AS\|_F^2 \ , \]
were $F$ denotes the Frobenius norm.  Many other tools have also been used to encourage additional sparsity of $S$, such as $l_1$ minimization and variants of matching pursuit \cite{GWO,SGO,IBP,Greer10}.  If no spatial correlations are assumed, the demixing problem can be solved at each pixel independently.  We can also add one of the nonconvex inter sparsity penalties defined by $H_0$ in (\ref{H}) or $S_0^{\epsilon}$ in (\ref{Sep}).  The resulting problem can be written in the form
\begin{equation}\label{eq:basic_demix} \min_{x_p \geq 0} \frac{1}{2}\|Ax_p - b_p\|^2 + R(x_p) \ , \end{equation}
where $x_p$ is the $p$th column of $S$ and $b_p$ is the $p$th column of $Y$.  We can define $R(x_p)$ to equal $H_0(x_p)$ or $S_0^{\epsilon}(x_p)$, putting (\ref{eq:basic_demix}) in the general form of (\ref{eq:generalprobX}).

\subsubsection{Structured Sparse Hyperspectral Demixing}
In hyperspectral demixing applications, the dictionary of endmembers is usually not known precisely.  There are many methods for learning endmembers from a hyperspectral image such as N-FINDR \cite{winter99}, vertex component analysis (VCA) \cite{Nascimento05}, NMF \cite{Pauca06}, Bayesian methods \cite{Zare08,Castrodad10} and convex optimization \cite{EMOSX}.  However, here we are interested in the case where we have a large library of measured reference endmembers including multiple references for each expected material measured under different conditions.  The resulting dictionary $A$ is assumed to have the group structure $[A_1, \cdots, A_M]$, where each group $A_j$ contains different references for the same $j$th material.

There are several reasons that we don't want to use the sparse demixing methods of Section \ref{sec:basic_demix} when $A$ contains a large library of references defined in this way.  Such a matrix $A$ with many nearly redundant references will likely have high coherence.  This creates a challenge for existing methods.  The grouped structure of $A$ also means that we want to enforce a structured sparsity assumption on the columns of $S$.  The linear combination of endmembers at any particular pixel is assumed to involve at most one endmember from each group $A_j$.  Linearly combining multiple references within a group may not be physically meaningful, since they all represent the same material.  Restricting our attention to a single pixel $p$, we can write the $p$th abundance column $x_p$ of $S$ as $ \bbm x_{1,p} \\ \vdots \\ x_{M,p} \ebm $.  The sparsity assumption requires each group of abundance coefficients $x_{j,p}$ to be at most one sparse.  We can enforce this by adding sufficiently large intra sparsity penalties to the objective in (\ref{eq:basic_demix}) defined by $H_j(x_{j,p})$ (\ref{H}) or $S_j^{\epsilon}(x_{j,p})$ (\ref{Sep}).

We think it may be important to use an expanded dictionary to allow different endmembers within groups to be selected at different pixels, thus incorporating endmember variability into the demixing process.  Existing methods accomplish this in different ways, such as the piece-wise convex endmember detection method in \cite{ZG}, which represents the spectral data as convex combinations of endmember distributions.  It is observed in \cite{ZG} that real hyperspectral data can be better represented using several sets of endmembers.  Additionally, their better performance compared to VCA, which assumes pixel purity, on a dataset which should satisfy the pixel purity assumption, further justifies the benefit of incorporating endmember variability when demixing.

If the same set of endmembers were valid at all pixels, we could attempt to enforce row sparsity of $S$ using for example the $l_{1,\infty}$ penalty used in \cite{EMOSX}, which would encourage the data at all pixels to be representable as non-negative linear combinations of the same small subset of endmembers.  Under some circumstances, this is a reasonable assumption and could be a good approach.  However, due to varying conditions, a particular reference for some material may be good at some pixels but not at others.  Although atmospheric conditions are of course unlikely to change from pixel to pixel, there could be nonlinear mixing effects that make the same material appear to have different spectral signatures in different locations \cite{KM}.  For instance, a nonuniform layer of dust will change the appearance of materials in different places.  If this mixing with dust is nonlinear, then the resulting hyperspectral data cannot necessarily be well represented by the linear mixture model with a dust endmember added to the dictionary.  In this case, by considering an expanded dictionary containing reference measurements for the materials covered by different amounts of dust, we are attempting to take into account these nonlinear mixing effects without explicitly modeling them.  At different pixels, different references for the same materials can now be used when trying to best represent the data.

The overall model should contain both intra and inter sparsity penalties.  In addition to the one sparsity assumption within groups, it is still assumed that many fewer than $M$ materials are present at any particular pixel.  The full model can again be written as (\ref{eq:basic_demix})
except with the addition of intra sparsity penalties.  The overall sparsity penalties can be written either as
\[R(x_p,d_p) = \sum_{j=1}^M \gamma_j H_j(x_{j,p},d_{j,p}) + \gamma_0 H_0(x_p) \]
or
\[ R(x_p) = \sum_{j=1}^M \gamma_j S_j^{\epsilon_j}(x_{j,p}) + \gamma_0 S_0^{\epsilon_0}(x_p) \ . \]




%% file: numerics.tex
\section{Numerical Experiments \label{numerics}}
In this section, we evaluate the effectiveness of our implementations of Problems 1 and 2 on the four applications discussed in Section \ref{applications}.  The simplest DOAS example with wavelength misalignment from Section \ref{sec:basic_DOAS} is used to see how well the intra sparsity assumption is satisfied compared to other methods.  Two convex methods that we compare to are NNLS (\ref{eq:NNLS}) and a non-negative constrained $l_1$ basis pursuit model like the template matching via $l_1$ minimization in \cite{guoO11}.  The $l_1$ minimization model we use here is
\begin{equation}\label{eq:L1}
\min_{x \geq 0} \|x\|_1 \qquad \text{such that} \qquad \|Ax-b\| \leq \tau \ . \end{equation}
We use MATLAB's \verb|lsqnonneg| function, which is parameter free, to solve the NNLS model.  We use Bregman iteration \cite{YOGD} to solve the $l_1$ minimization model.  We also compare to direct $l_0$ minimization via penalty decomposition (Algorithm
\ref{alg:pdL0}).

The penalty decomposition method \cite{luZ12} amounts to solving
(\ref{eq:opt_constraint}) by a series of minimization problems with
an increasing sequence $\{\rho_k\}$. Let $x = [\bx_1, \cdots, \bx_M]$, $y=[\by_1, \cdots,\by_M]$ and iterate 
\begin{equation}
 \left. \begin{array}{l}
(x^{k+1}, y^{k+1}) = \mbox{arg}\min \frac 1 2 \|Ax-b\|^2+\frac
{\rho_k}2\|x-y\|^2\\
\qquad \qquad \qquad \quad \mbox{s.t.} \quad \by_j\geqslant 0, \ \|\by_j\|_0 \leqslant 1 \\
 \rho^{k+1} = \sigma\rho^k \quad (\mbox{for}\ \sigma>1) \ . 
  \end{array}\right.
\end{equation}
The pseudo-code of this method is given in Algorithm
\ref{alg:pdL0}.

\alg{ \label{alg:pdL0}
A penalty decomposition method for solving (\ref{eq:opt_constraint}) \\ \ \\
\noindent Define $\rho>0, \sigma>1, \epsilon_o, \epsilon_i$ and initialize $y$. \\

\texttt{while} \  $\|x - y\|_{\infty} > \epsilon_o$

    \qquad i = 1;

    \qquad \texttt{while} $\max\{\|x^i-x^{i-1}\|_\infty, \|y^i-y^{i-1}\|_\infty\} > \epsilon_i$

    \qquad \qquad $x^i = (A^TA+\rho Id)^{-1} (A^Tb+\rho y^i)$

    \qquad \qquad $y^i = 0$

    \qquad \qquad \texttt{for}  $j = 1, \cdots, M$

    \qquad \qquad \qquad find the index of maximal $\bx_j$,
    \emph{i.e.}, $l_j = \arg \max_l \bx_j(l)$

    \qquad \qquad \qquad Set $\by_j(l_j)=\max(\bx_j(l_j),0)$

    \qquad \qquad \texttt{end for}

    \qquad \qquad i = i+1;

    \qquad \texttt{end while}

    \qquad $x = x^i, y=y^i, \rho = \sigma \rho$

\texttt{end while} }

Algorithm \ref{alg:pdL0} may require a good initialization of $y$ or a slowly increasing $\rho$.  If the maximum magnitude locations within each group are initially incorrect, it can get stuck at a local minimum.  We consider both least square (LS) and NNLS initializations in numerical experiments.  Algorithms \ref{SGPalg_dynamic} and \ref{SGPalg} also benefit from a good initialization for the same reason.  We use a constant initialization, for which the first iteration of those methods is already quite similar to NNLS.

We also test the effectiveness of Problems 1 and 2 on the three other applications discussed in Section \ref{applications}.  For DOAS with the included background model, we compare again to Algorithm \ref{alg:pdL0}.  We use the sparse hyperspectral demixing example to demonstrate the sparsifying effect of the inter sparsity penalties acting without any intra sparsity penalties.  We compare to the $l_1$ regularized demixing model in \cite{GWO} using the implementation in \cite{SGO}. To illustrate the effect of the intra and inter sparsity penalties acting together, we also apply Problems 1 and 2 to a synthetic example of structured sparse hyperspectral demixing.  We compare the recovery of the ground truth abundance with and without the intra sparsity penalties.

\subsection{DOAS with Wavelength Alignment \label{sc:DOAS_wavelength}}
We generate the dictionary by taking three given reference spectra $y_j(\lambda)$ for the gases HONO, NO2 and O3 and deforming each by a set of linear
functions.  The resulting dictionary contains $ y_j(\lambda+P_k\lambda+Q_l)$ for $P_k =
-1.01+0.01k$ ($k = 1, \cdots, 21$), $Q_l = -1.1+0.1l$ ($l = 1,
\cdots, 21$) and $j = 1,2,3$.  Each $y_j \in \R^W$ with  $W = 1024$.  The represented wavelengths in nanometers are $\lambda = 340 + 0.04038w, \  w = 0,..,1023$.  We use odd reflections to extrapolate shifted references at the boundary.  The choice of boundary condition should only have a small effect if the wavelength displacements are small.  However, if the displacements are large, it may be a good idea to modify the data fidelity term to select only the middle wavelengths to prevent boundary artifacts from influencing the results.

There are a total of $441$ linearly deformed references for each of the three groups. In Figure~\ref{fig:DOAS_spectra}, we plot the
reference spectra of HONO, NO2 and O3 together with several deformed
examples.

\begin{figure}
\begin{center}
\begin{tabular}{ccc}
HONO & NO2 & O3\\
\includegraphics[width=0.3\textwidth]{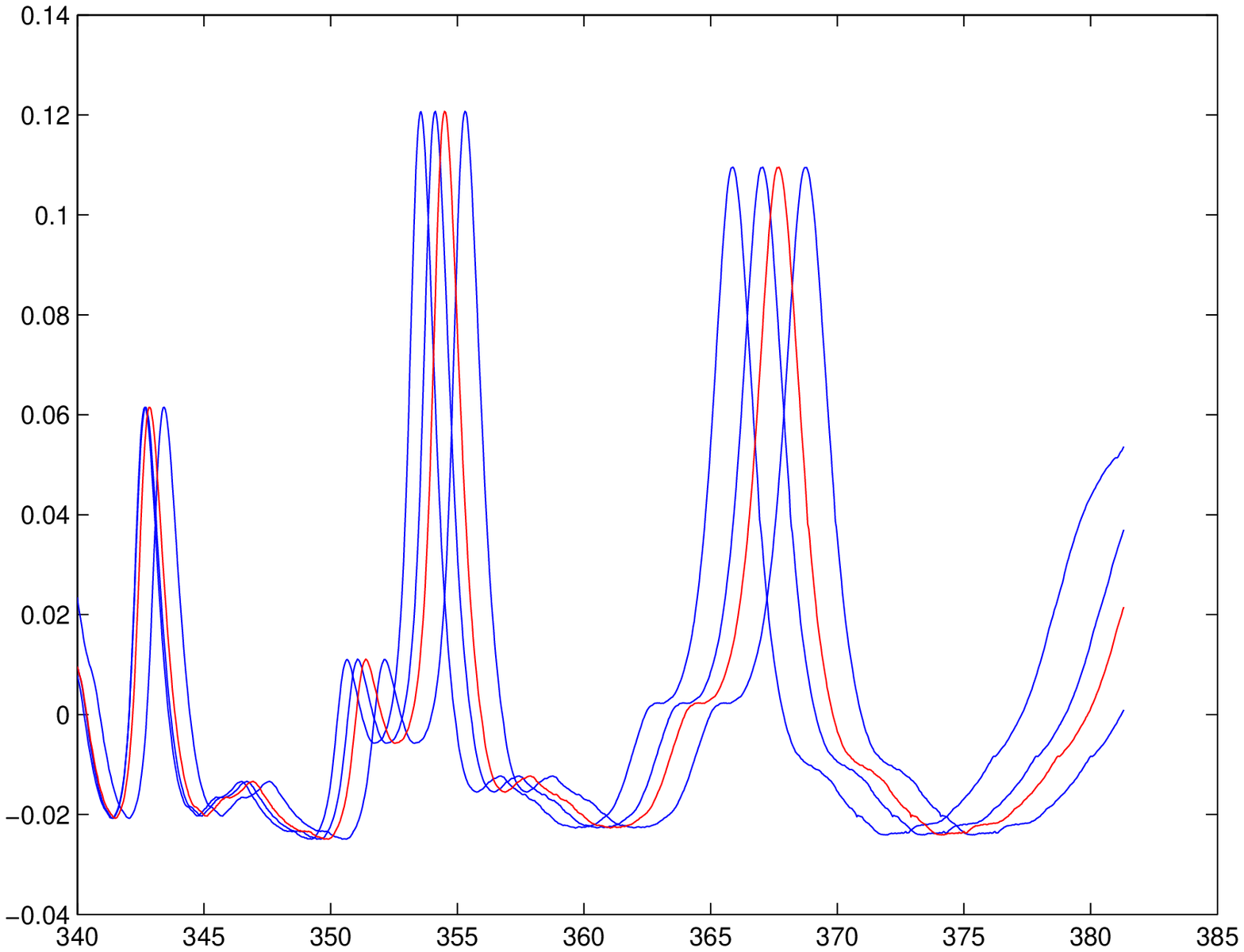}&
\includegraphics[width=0.3\textwidth]{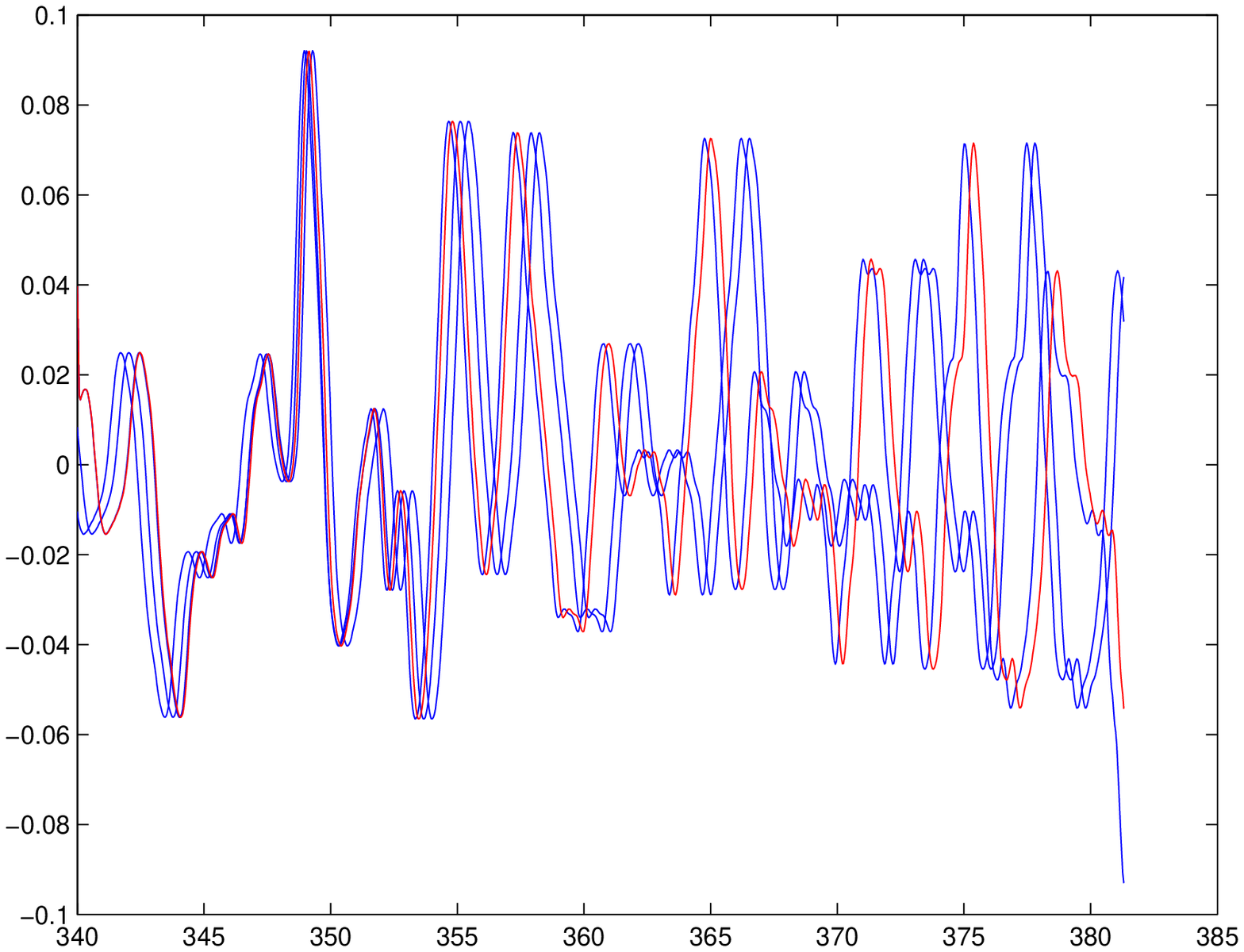}&
\includegraphics[width=0.3\textwidth]{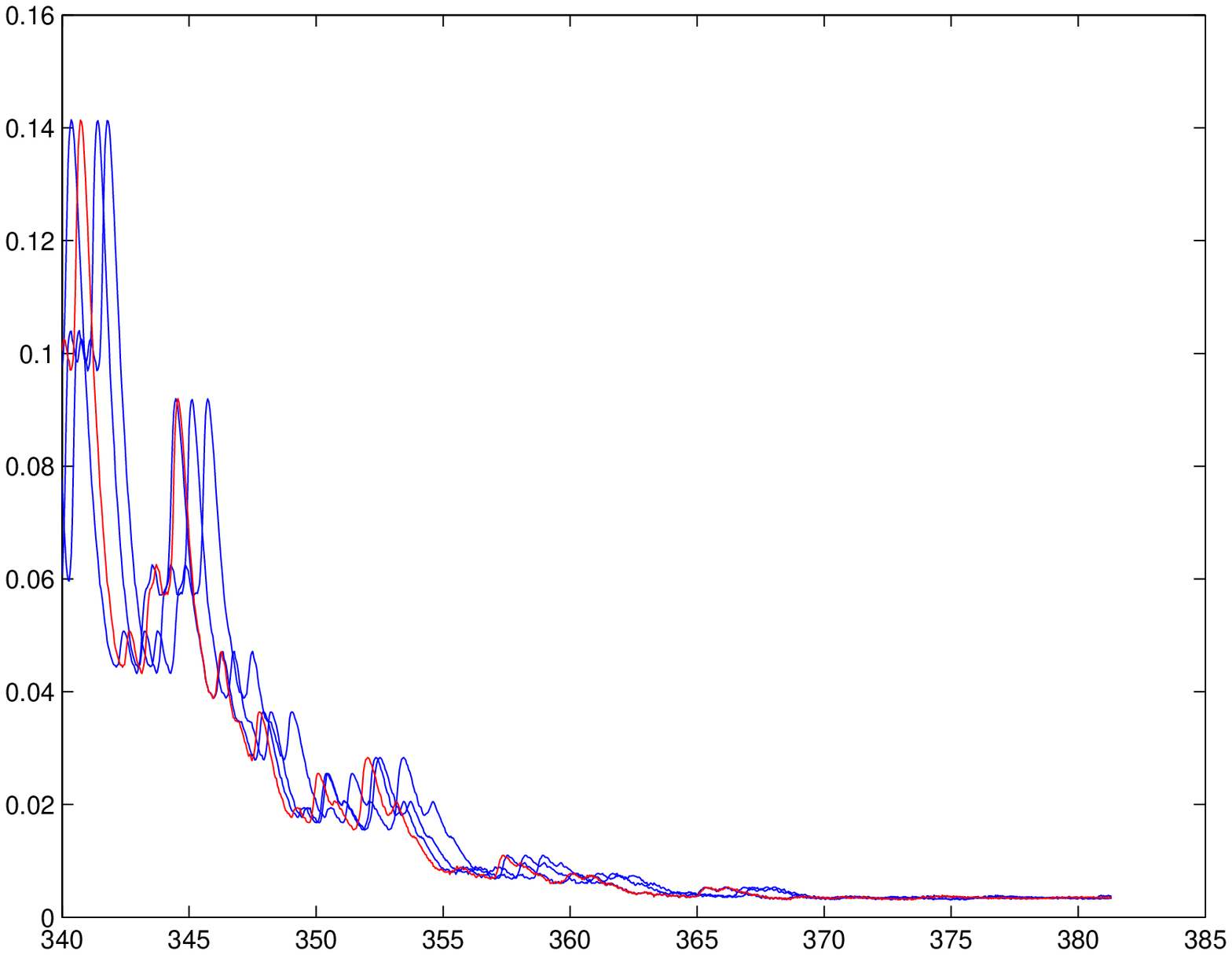}\\
\end{tabular}
\caption{For each gas, the reference spectrum is plotted in red,
while three deformed spectra are in blue. } \label{fig:DOAS_spectra}
\end{center}
\end{figure}

In our experiments, we randomly select one element for each group
with random magnitude plus additive zero mean Gaussian noise to synthesize the data
term $J(\lambda) \in \R^{W}$ for $W=1024$.  Mimicking the relative magnitudes of a real DOAS dataset \cite{DOASdata} after normalization of the dictionary, the random
magnitudes are chosen to be at different orders with mean values of 1, 0.1, 1.5
for HONO, NO2 and O3 respectively.  We perform three experiments for which the standard deviations of the noise are $0$, $.005$ and $.05$ respectively.  This synthetic data is shown in Figure \ref{fig:DOAS_data}.

\begin{figure}
\begin{center}
\begin{tabular}{ccc}
no noise & $\sigma = .005$ & $\sigma = .05$ \\
\includegraphics[width=0.3\textwidth]{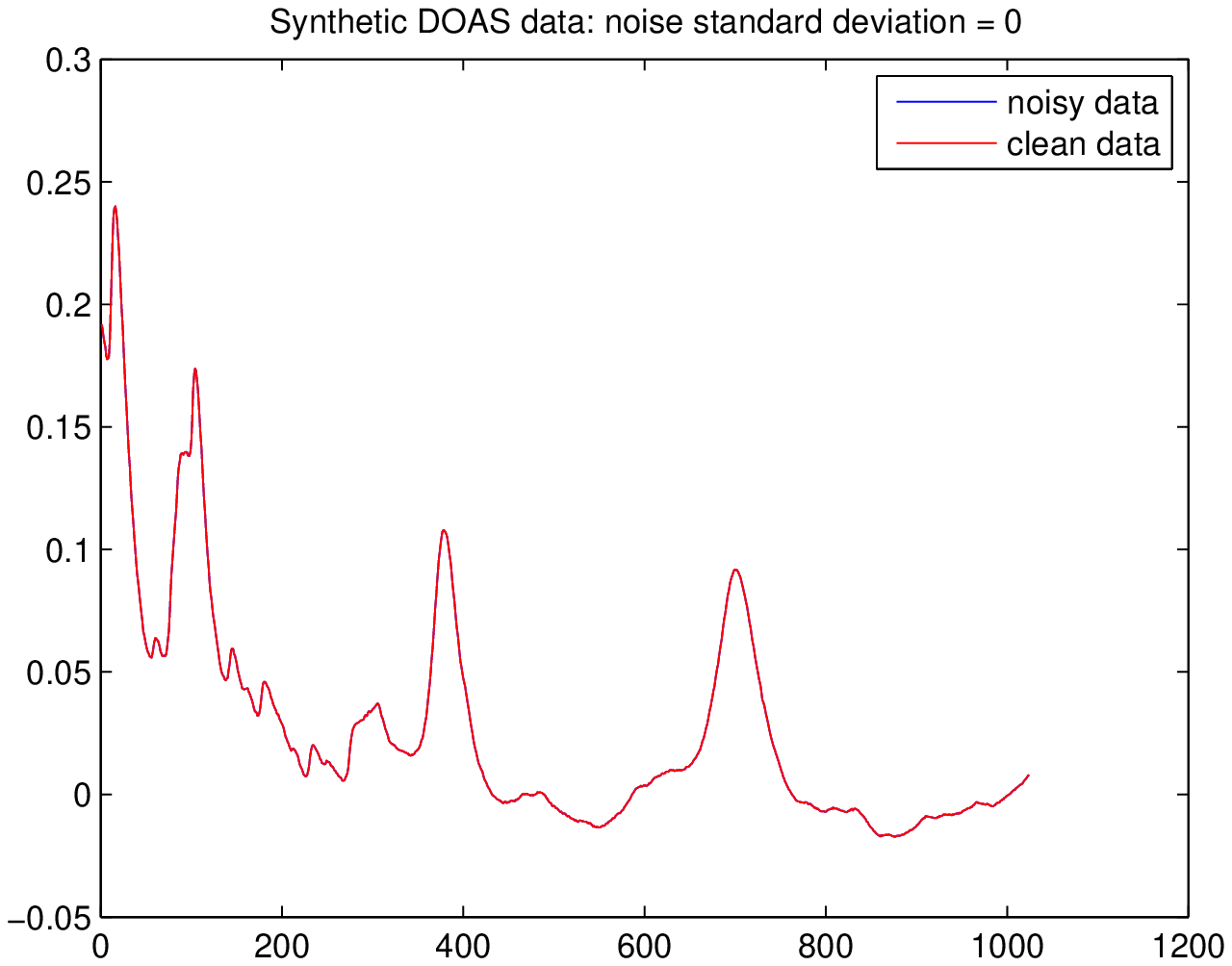}&
\includegraphics[width=0.3\textwidth]{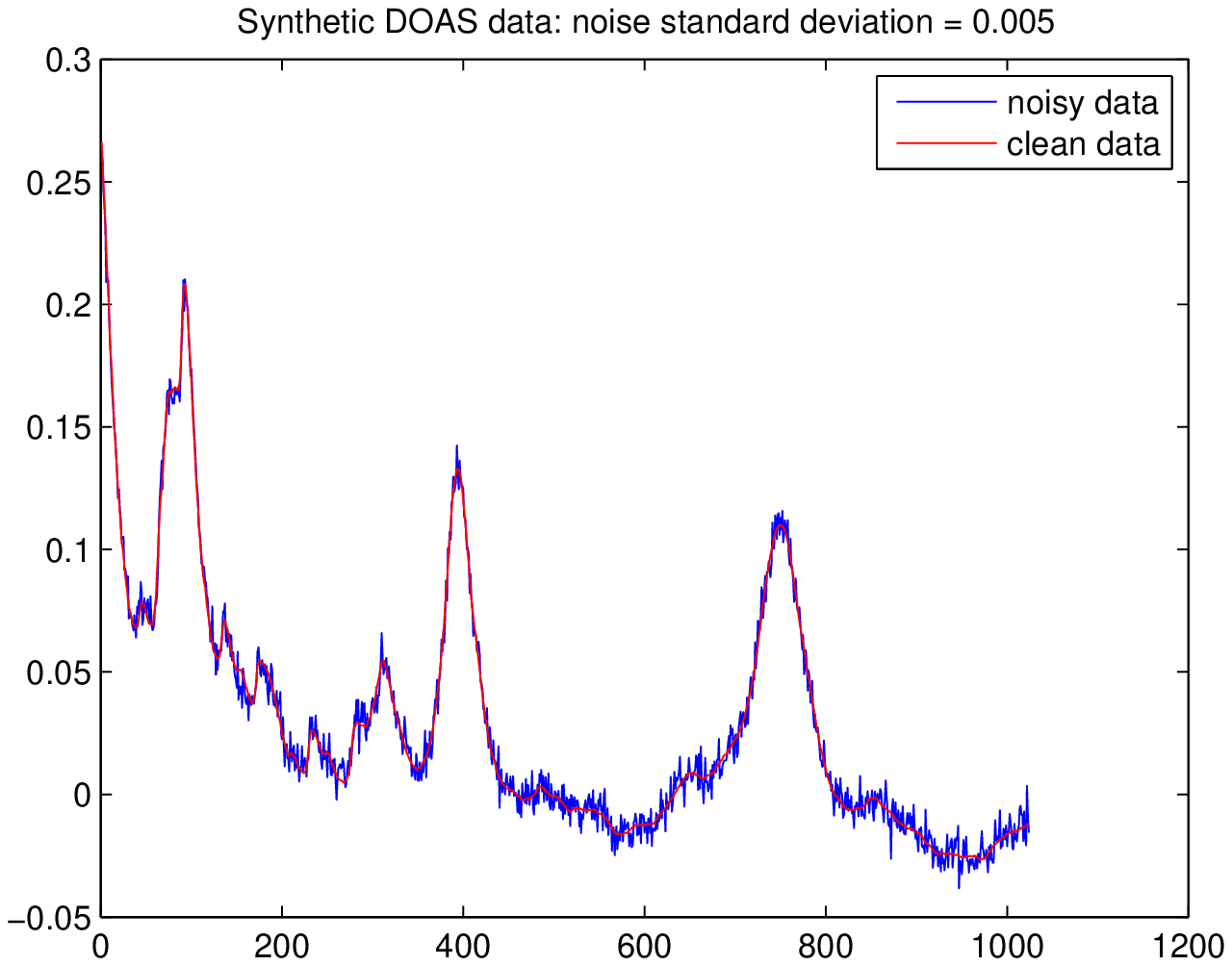}&
\includegraphics[width=0.3\textwidth]{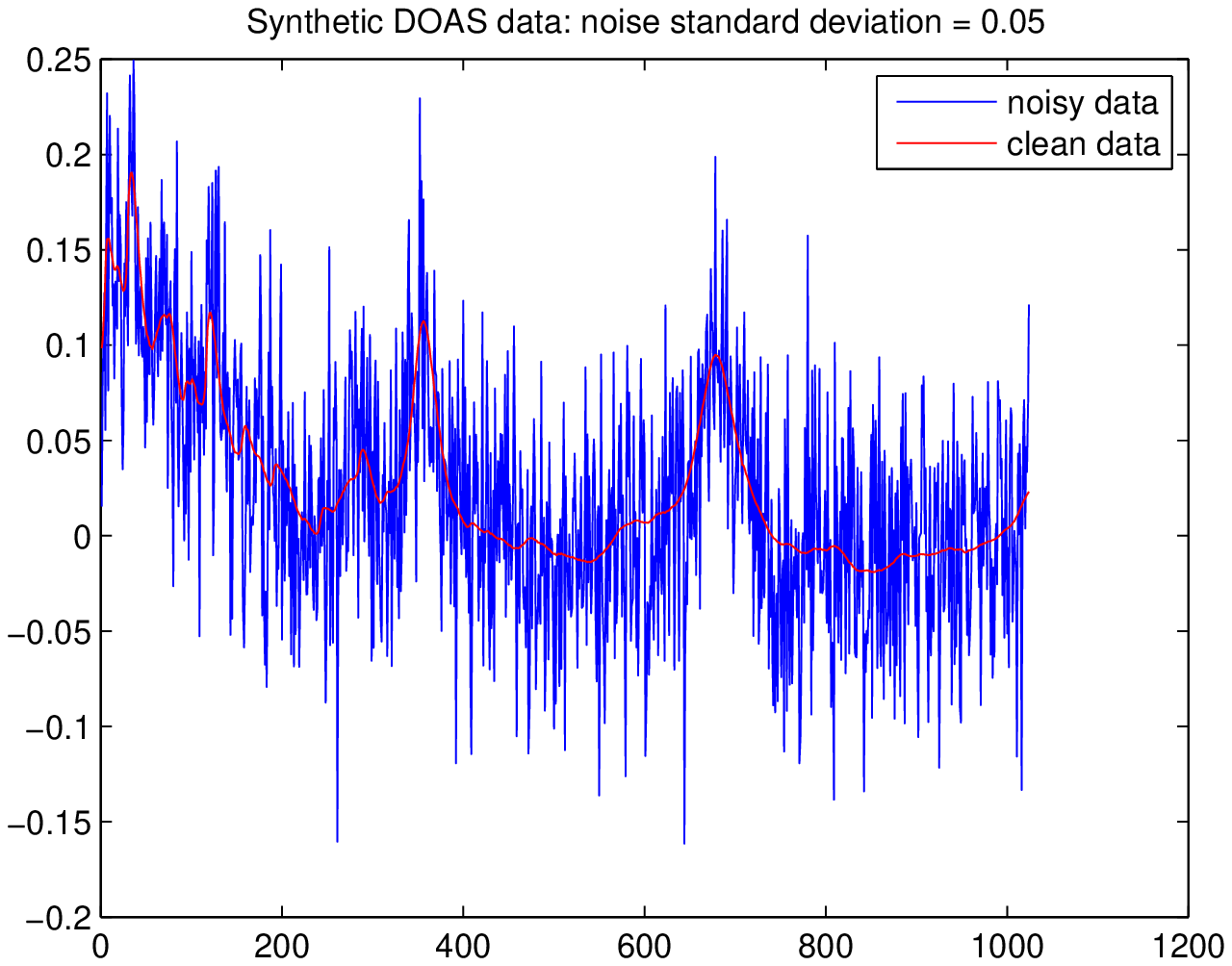}\\
\end{tabular}
\caption{Synthetic DOAS data } \label{fig:DOAS_data}
\end{center}
\end{figure}

The parameters used in the numerical experiments are as follows.  NNLS is parameter free.  For the $l_1$ minimization method in (\ref{eq:L1}), $\frac{\tau}{\sqrt{W}} = .001$, $.005$ and $.05$ for the experiments with noise standard deviations of $0$, $.005$ and $.05$ respectively.  For the direct $l_0$ method (Algorithm \ref{alg:pdL0}), the penalty parameter $\rho$ is initially equal to $.05$ and increases by a factor of $\sigma = 1.2$ every iteration.  The inner and outer tolerances are set at $10^{-4}$ and $10^{-5}$ respectively.  The initialization is chosen to be either a least squares solution or the result of NNLS.  For Problems 1 and 2 we define $\epsilon_j = .05$ for all three groups.  In general this could be chosen roughly on the order of the smallest nonzero coefficient expected in the $j$th group. Recall that these $\epsilon_j$ are used both in the definitions of the regularized $l_1$ - $l_2$ penalties $S_j^{\epsilon}$ in Problem 2 and in the definitions of the dummy variable constraints in Problem 1.  We set $\gamma_j = .1$ and $\gamma_j = .05$ for Problems 1 and 2 respectively and for $j = 1,2,3$.  Since there is no inter sparsity penalty, $\gamma_0 = 0$.  For both Algorithms \ref{SGPalg_dynamic} and \ref{SGPalg} we set $C = 10^{-9}\I$.  For Algorithm \ref{SGPalg_dynamic}, which dynamically updates $C$, we set several additional parameters $\sigma = .1$, $\xi_1 = 2$ and $\xi_2 = 10$.  These choices are not crucial and have more to do with the rate of convergence than the quality of the result.  For both algorithms, the outer iterations are stopped when the difference in energy is less than $10^{-8}$, and the inner ADMM iterations are stopped when the relative errors of the primal and dual variables are both less than $10^{-4}$.

We plot results of the different methods in blue along with the ground truth solution in red.  The experiments are shown in Figures \ref{fig:DOAS0}, \ref{fig:DOAS005} and \ref{fig:DOAS05}.



\begin{figure}
\begin{center}
\includegraphics[width=\textwidth]{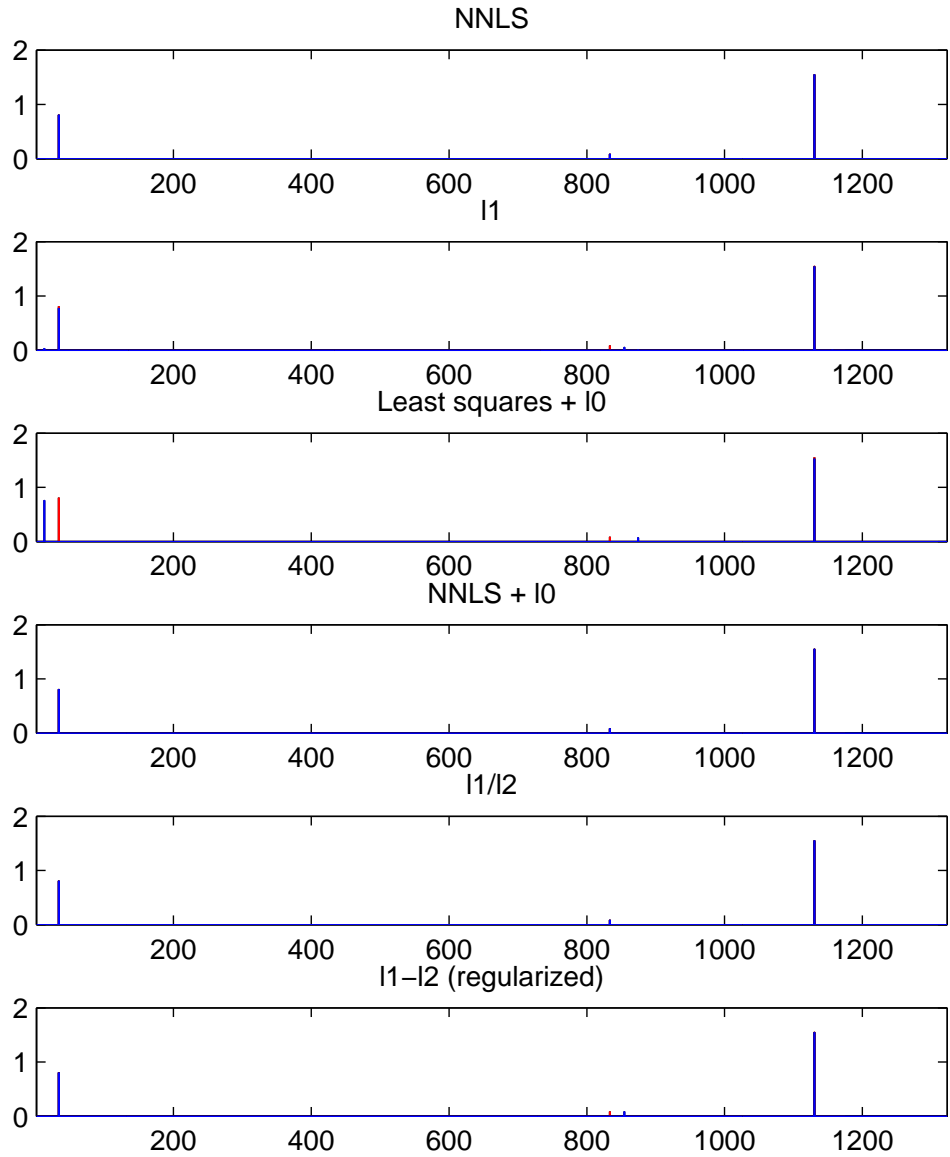} \\
\caption{Method comparisons on synthetic DOAS data without noise.  Computed coefficients (blue) are plotted on top of the ground truth (red).} \label{fig:DOAS0}
\end{center}
\end{figure}

\begin{figure}
\begin{center}
\includegraphics[width=\textwidth]{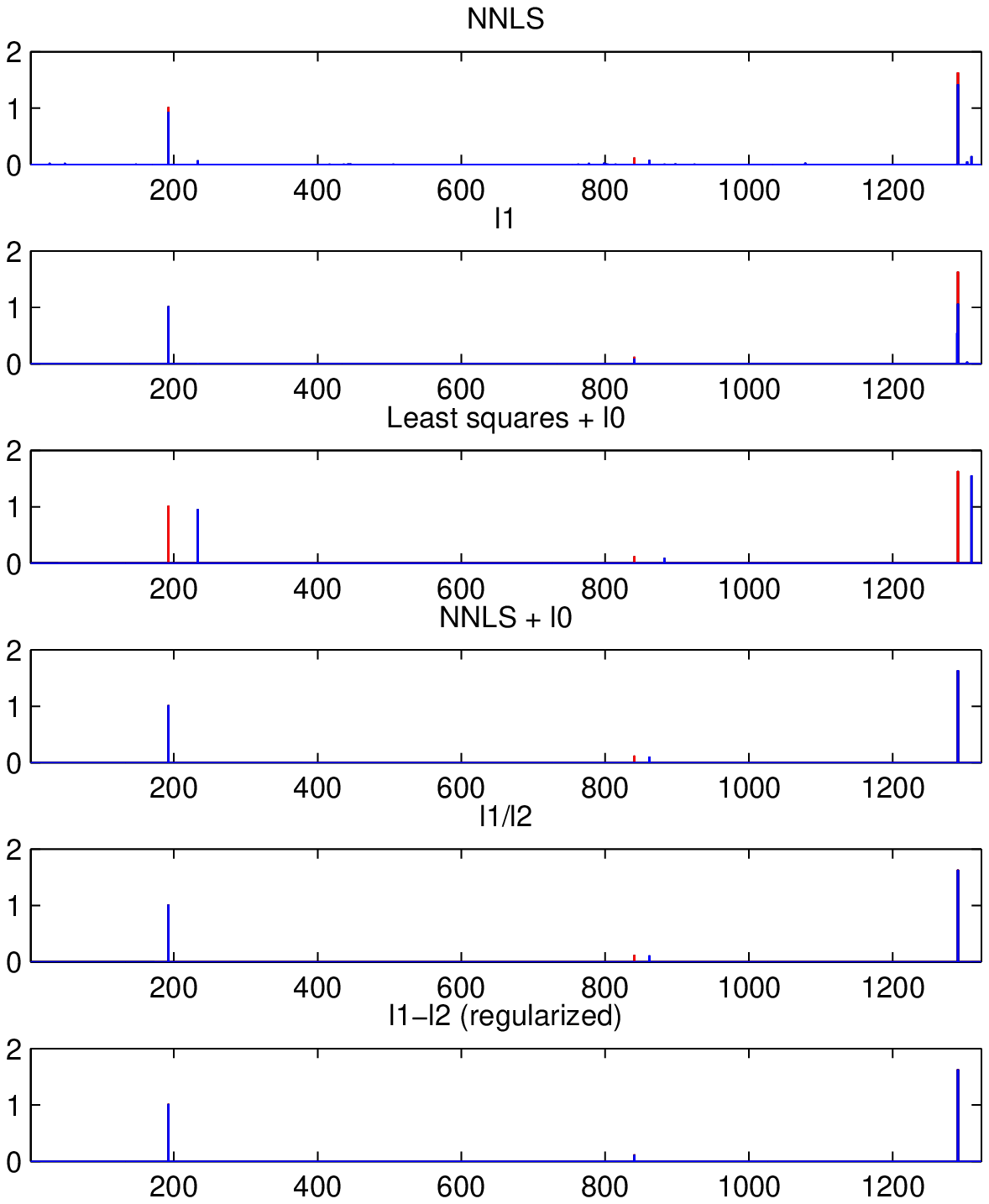} \\
\caption{Method comparisons on synthetic DOAS data: $\sigma = .005$. Computed coefficients (blue) are plotted on top of the ground truth (red).} \label{fig:DOAS005}
\end{center}
\end{figure}

\begin{figure}
\begin{center}
\includegraphics[width=\textwidth]{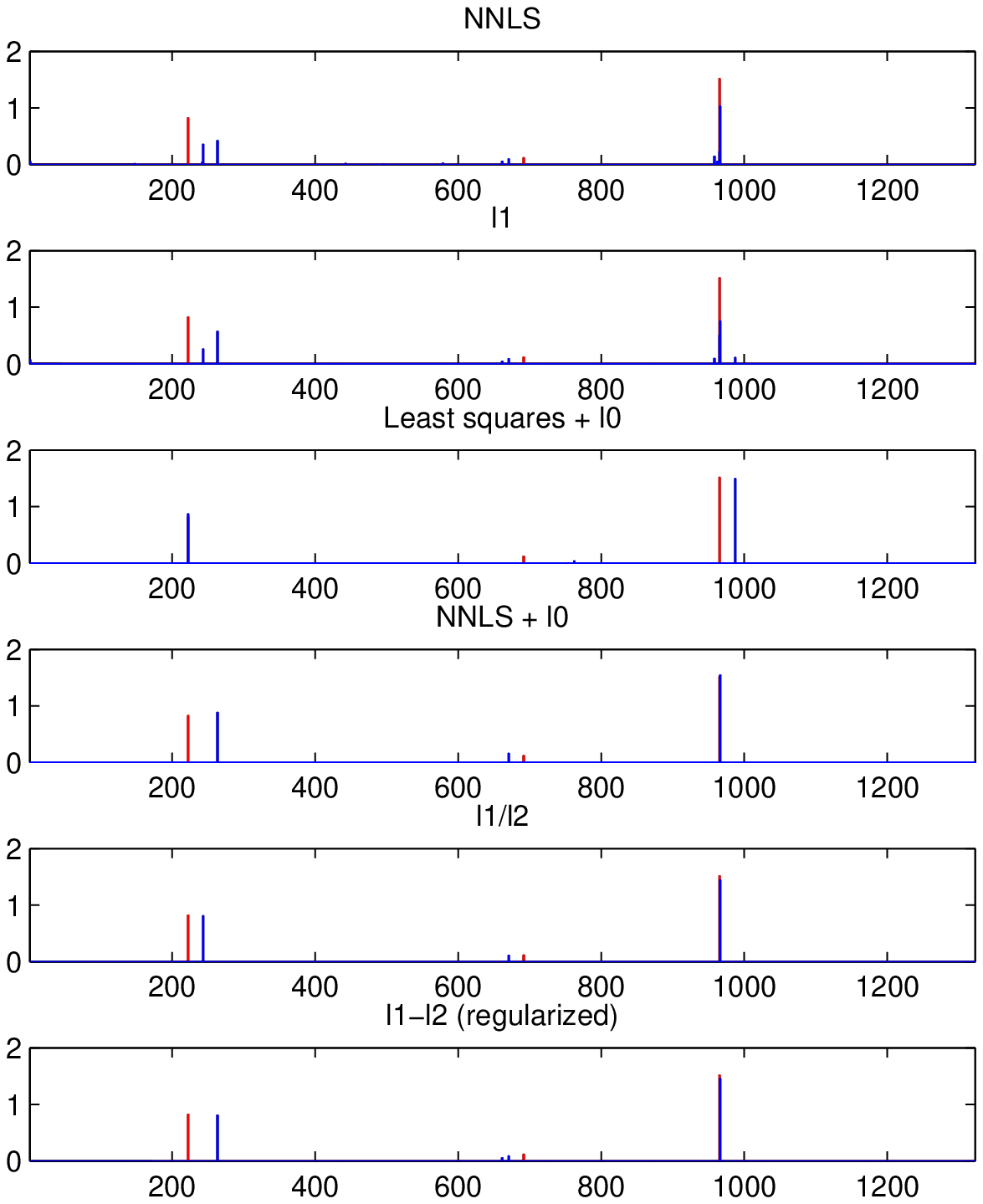} \\
\caption{Method comparisons on synthetic DOAS data: $\sigma = .05$. Computed coefficients (blue) are plotted on top of the ground truth (red).} \label{fig:DOAS05}
\end{center}
\end{figure}

\subsection{DOAS with Wavelength Alignment and Background Estimation}
We solve the model (\ref{eq:opt_background}) using $l_1$/$l_2$ and regularized $l_1$ - $l_2$ intra sparsity penalties.  These are special cases of Problems 1 and 2 respectively.  Depending on which, the convex set $X$ is either the non-negative orthant or a subset of it.  We compare the performance to the direct $l_0$ method (Algorithm \ref{alg:pdL0}) and least squares.  The dictionary consists of the same set of linearly deformed reference spectra for HONO, NO2 and O3 as in Section \ref{sc:DOAS_wavelength}.  The data $J$ is synthetically generated by
\[ J(\lambda) = .0121y_1(\lambda) + .0011y_2(\lambda) + .0159y_3(\lambda) + \frac{2}{(\lambda - 334)^4} + \eta(\lambda), \]
where the references $y_j$ are drawn from columns $180$, $682$ and $1103$ of the dictionary and the last two terms represent a smooth background component and zero mean Gaussian noise having standard deviation $5.58 10^{-5}$.  The parameter $\alpha$ in (\ref{eq:opt_background}) is set at $10^{-5}$ for all the experiments.

The least squares method for (\ref{eq:opt_background}) directly solves
\[ \min_{x, B} \frac{1}{2}\left\| \bbm A_3 & \I \\ 0 & \sqrt{\alpha}Q \ebm \bbm x \\ B \ebm - \bbm J \\ 0 \ebm \right\|^2 \ , \]
where $A_3$ has only three columns randomly chosen from the expanded dictionary $A$, with one chosen from each group.  Results are averaged over $1000$ random selections.

In Algorithm \ref{alg:pdL0}, the penalty parameter $\rho$ starts at $10^{-6}$ and increases by a factor of $\sigma = 1.1$ every iteration.  The inner and outer tolerances are set at $10^{-4}$ and $10^{-6}$ respectively.  The coefficients are initialized to zero.

In Algorithms \ref{SGPalg_dynamic} and \ref{SGPalg}, we treat the background as a fourth group of coefficients, after the three for each set of reference spectra.  For all groups $\epsilon_j$ is set to $.001$.  We set $\gamma_j = .001$ for $j = 1,2,3$, and $\gamma_4 = 0$, so no sparsity penalty is acting on the background component.  We set $C = 10^{-7} \I$ for Algorithm \ref{SGPalg} and $C = 10^{-4} \I$ for Algorithm \ref{SGPalg_dynamic}, where again we use $\sigma = .1$, $\xi_1 = 2$ and $\xi_2 = 10$.  We use a constant but nonzero initialization for the coefficients $x$.  The inner and outer iteration tolerances are the same as in Section \ref{sc:DOAS_wavelength} with the inner decreased to $10^{-5}$.

Figure \ref{fig:fit} compares how closely the results of the four methods fit the data.  Plotted are the synthetic data, the estimated background, each of the selected three linearly deformed reference spectra multiplied by their estimated fitting coefficients and finally the sum of the references and background.
\begin{figure}
\begin{center}
\includegraphics[width=\textwidth]{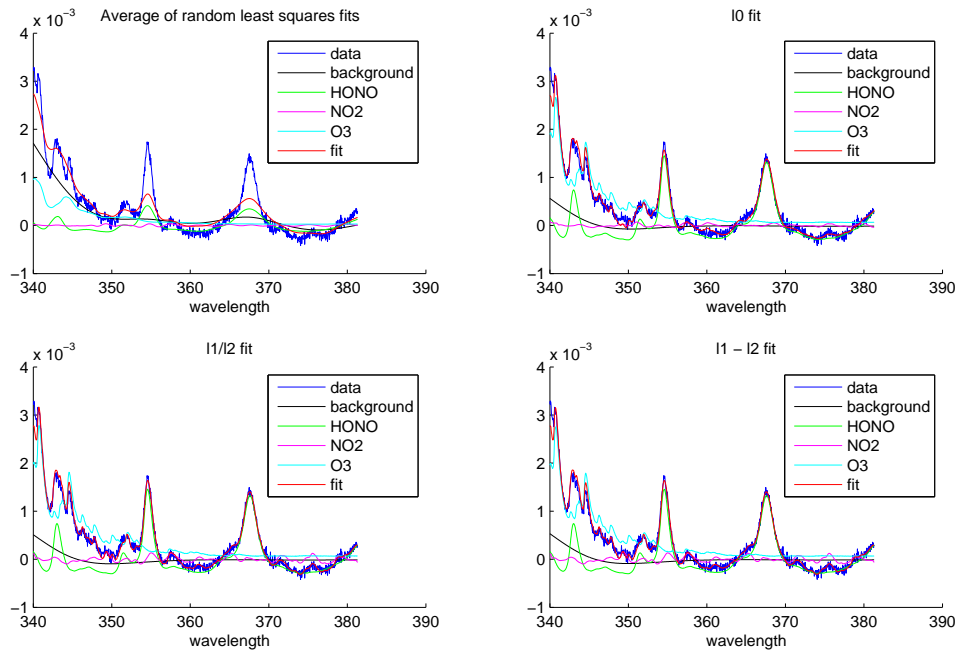} \\
\caption{Comparisons of how well the results of least squares, direct $l_0$, $l_1$/$l_2$ and regularized $l_1$ - $l_2$ fit the data.} \label{fig:fit}
\end{center}
\end{figure}

The computed coefficient magnitudes and displacements are compared to the ground truth in Table \ref{coeffsNdisps}.
\begin{table}
\begin{center}
\begin{tabular}{|l|c|c|c|c|c|}
\hline
& ground truth & least squares & $l_0$ & $l_1$/$l_2$ & $l_1$ - $l_2$ \\
\hline
$a_1$ (HONO coefficient) & 0.01206 & 0.00566 & 0.01197 & 0.01203 & 0.01202 \\
$a_2$ (NO2 coefficient) & 0.00112 & 0.00020 & 0.00081 & 0.00173 & 0.00173 \\
$a_3$ (O3 coefficient) & 0.01589 & 0.00812 & 0.01884 & 0.01967 & 0.01947 \\
\hline
$v_1$ (HONO displacement) & 0.01$\lambda$ - 0.2 & N/A & 0.01$\lambda$ - 0.2 & 0.01$\lambda$ - 0.2 & 0.01$\lambda$ - 0.2 \\
$v_2$ (NO2 displacement) & -0.01$\lambda$ + 0.1 & N/A & -0.09$\lambda$ - 0.9 & 0$\lambda$ - 0.2 & 0$\lambda$ - 0.2 \\
$v_3$ (O3 displacement) & 0$\lambda$ + 0 & N/A & 0$\lambda$ + 0 & 0$\lambda$ + 0  & 0$\lambda$ + 0  \\
\hline
\end{tabular}
\caption{Comparison of estimated fitting coefficients and displacements for DOAS with background estimation} \label{coeffsNdisps}
\end{center}
\end{table}

The dictionary perhaps included some unrealistically large deformations of the references.  Nonetheless, the least squares result shows that the coefficient magnitudes are underestimated when the alignment is incorrect.  The methods for the $l_0$, $l_1$/$l_2$ and regularized $l_1$ - $l_2$ models all produced good and nearly equivalent results.  All estimated the correct displacements of HONO and O3, but not NO2.  The estimated amounts of HONO and NO2 were correct.  The amount of O3 was overestimated by all methods.  This is because there was a large background component in the O3 reference.  Even with background estimation included in the model, it should still improve accuracy to work with references that have been high pass filtered ahead of time.

Although the methods for the $l_0$, $l_1$/$l_2$ and regularized $l_1$ - $l_2$ models all yielded similar solutions, they have different pros and cons regarding parameter selection and runtime.  It is important that $\rho$ not increase too quickly in the direct $l_0$ method.  Otherwise it can get stuck at a poor solution.  For this DOAS example, the resulting method required about 200 iterations and a little over 10 minutes to converge. Algorithm \ref{SGPalg_dynamic} for the $l_1$/$l_2$ model can sometimes waste effort finding splitting coefficients that yield a sufficient decrease in energy.  Here it required 20 outer iterations and ran in a few minutes.  Algorithm \ref{SGPalg} required 8 outer iterations and took about a minute.  Choosing $\gamma_j$ too large can also cause the $l_1$/$l_2$ and $l_1$ - $l_2$ methods to get stuck at bad local minima.  On the other hand, choosing $\gamma_j$ too small may result in the group 1-sparsity condition not being satisfied, whereas it is satisfied by construction in the direct $l_0$ approach.  Empirically, gradually increasing $\gamma_j$ works well, but we have simply used fixed parameters for all our experiments.

\comments{
Preliminary numerical results using $l_1$/$l_2$ are shown in Figure \ref{fig:DOAS_BG}
\begin{figure}
\begin{center}
\begin{tabular}{cc}
Background & Coefficients \\
\includegraphics[width=0.5\textwidth]{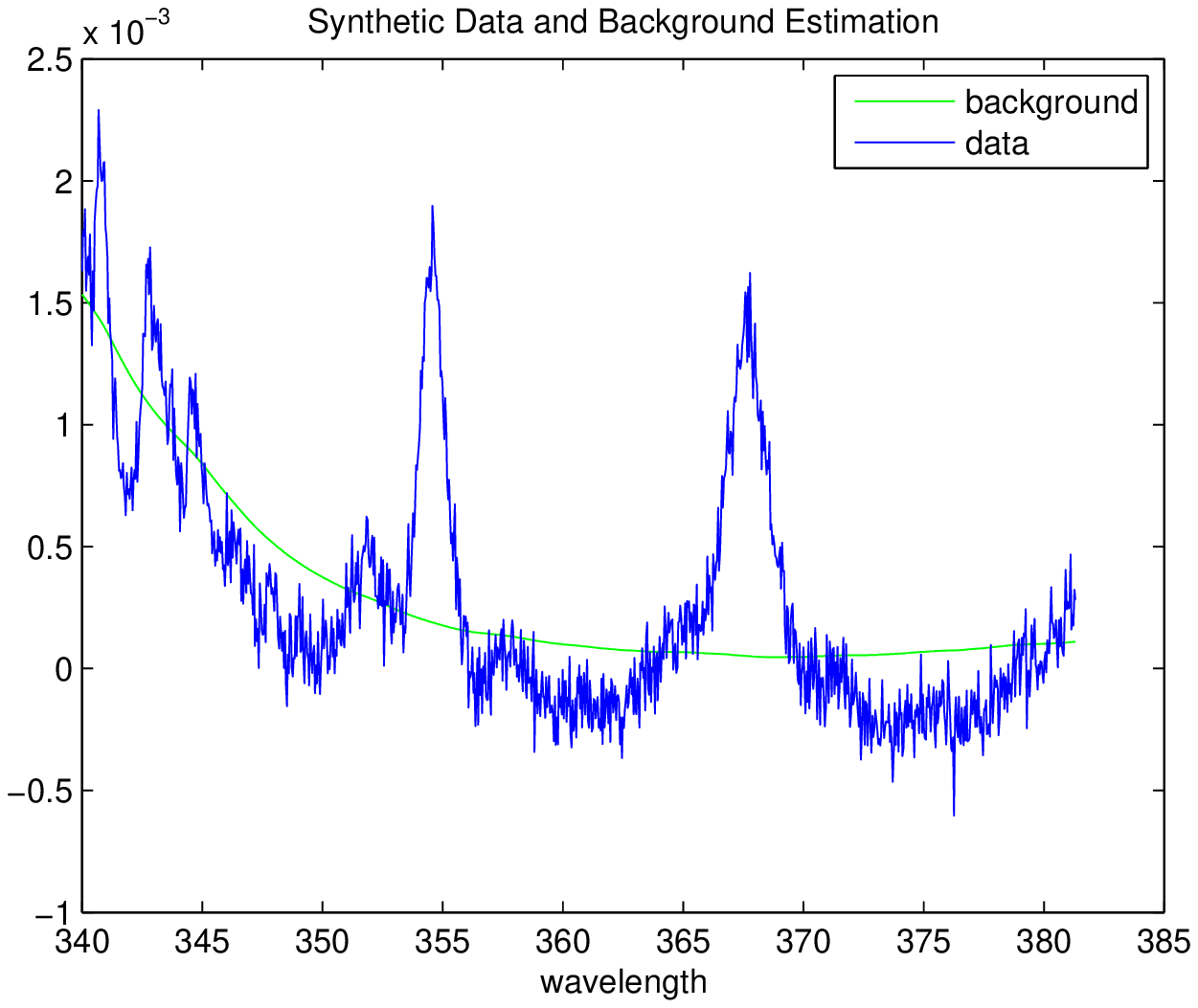}&
\includegraphics[width=0.5\textwidth]{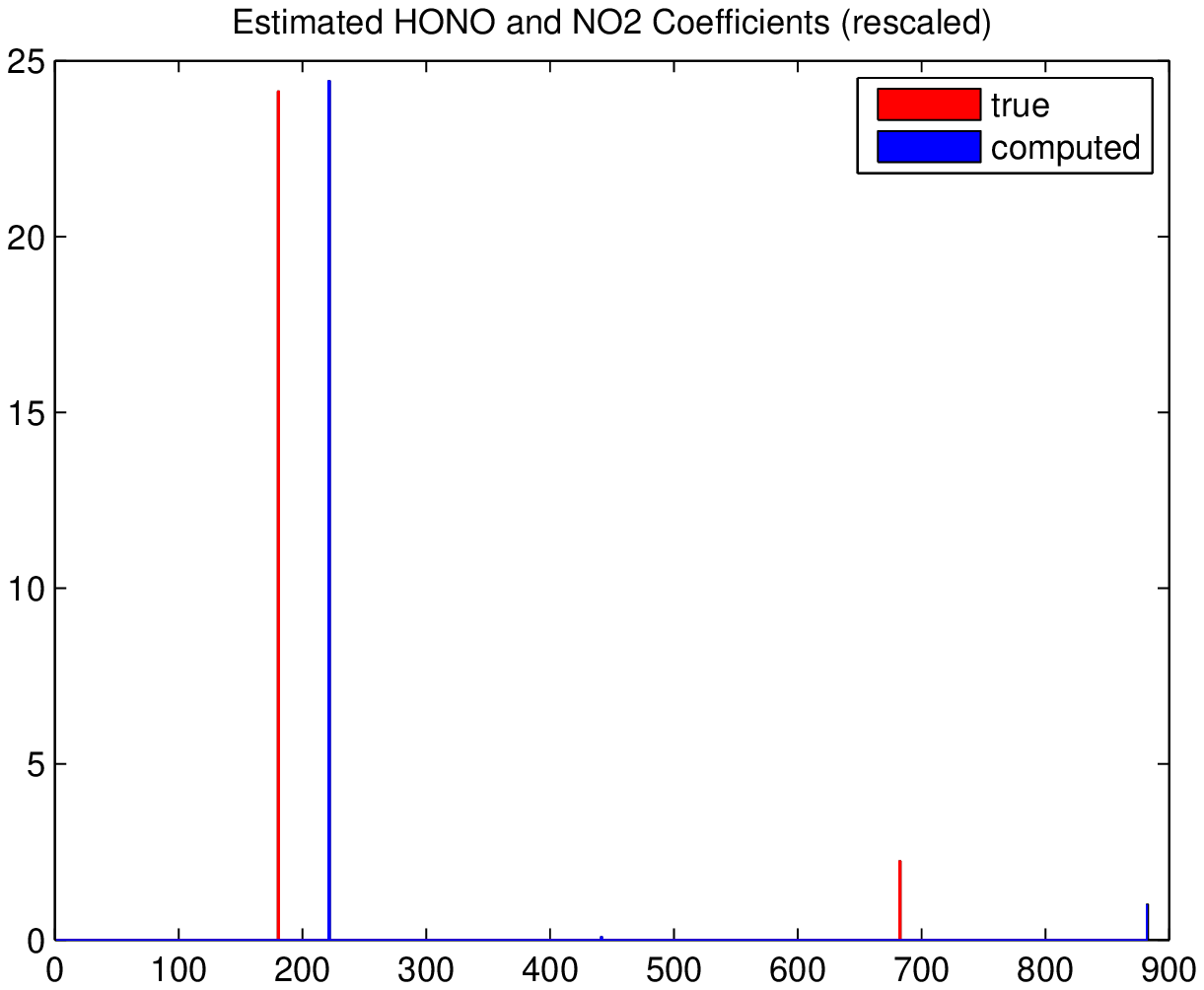}\\
\end{tabular}
\caption{Preliminary results for DOAS with background using $l_1$/$l_2$ sparsity penalty} \label{fig:DOAS_BG}
\end{center}
\end{figure}
}

\subsection{Hyperspectral Demixing with Inter Sparsity Penalty}
We use the urban hyperspectral dataset from \cite{urbdata}.  Each column of the data matrix $Y \in \R^{187 \times 94249}$ represents the spectral signature measured at a pixel in the 307 by 307 urban image shown in Figure \ref{urbanpic}.
\begin{figure}
\begin{center}
\begin{tabular}{cc}
\includegraphics[width=0.4\textwidth]{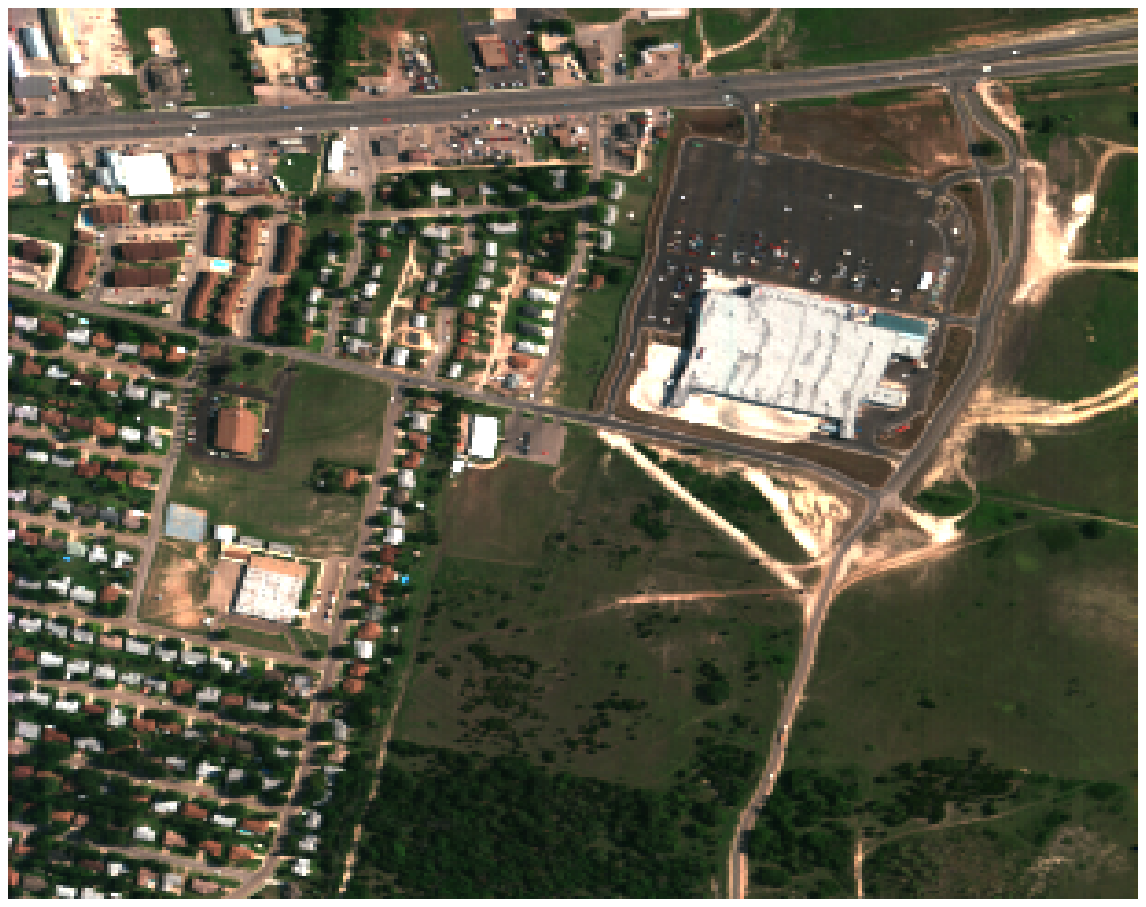}&
\includegraphics[width=0.4\textwidth]{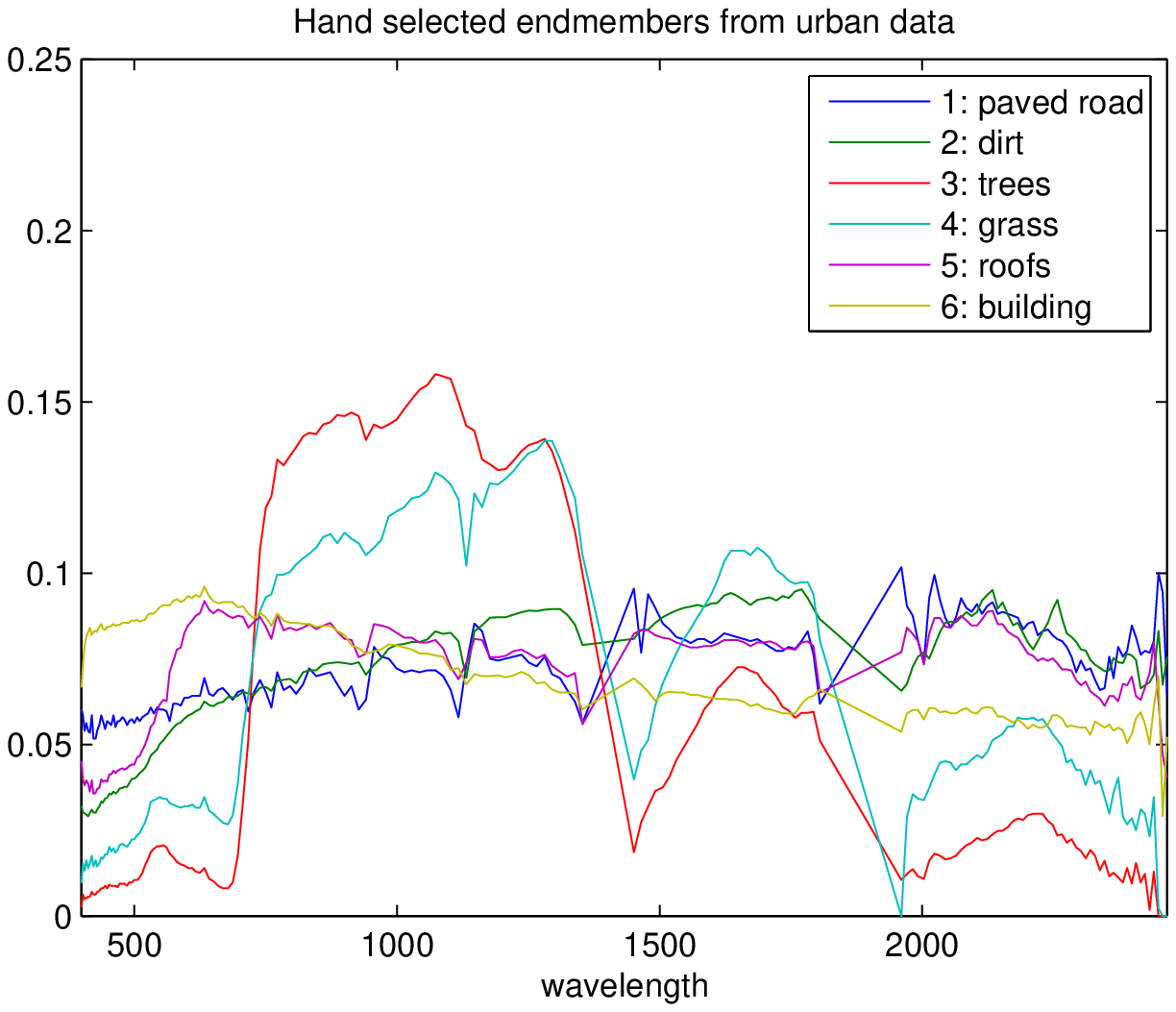}\\
\end{tabular}
\caption{Color visualization of urban hyperspectral image and hand selected endmembers } \label{urbanpic}
\end{center}
\end{figure}

The data was processed to remove some wavelengths for which the data was corrupted, resulting in a spectral resolution reduced from 210 to 187.  The six endmembers forming the columns of the dictionary $A$ were selected by hand from pixels that appeared to be pure materials.  These are also shown in Figure \ref{urbanpic}.  The columns of both $A$ and $Y$ were normalized to have unit $l_2$ norm.

Algorithms \ref{SGPalg_dynamic} and \ref{SGPalg} were used to solve (\ref{eq:basic_demix}) with $l_1$/$l_2$ and regularized $l_1$ - $l_2$ inter sparsity penalties respectively.  These were compared to NNLS and $l_1$ minimization \cite{SGO}, which solve
\begin{equation}\label{eq:l1demix} \min_{x_p \geq 0} \frac{1}{2}\|Ax_p - b_p\|^2 + \gamma \|x_p\|_1 \end{equation}
for each pixel $p$.  The parameters were chosen so that the $l_1$, $l_1$/$l_2$ and $l_1$ - $l_2$ approaches all achieved roughly the same level of sparsity, measured as the fraction of nonzero abundances.  The sparsity and sum of squares errors achieved by the four models are tabulated in Table \ref{density_sos}.
\begin{table}
\begin{center}
\begin{tabular}{|l|c|c|c|c|}
\hline
& NNLS & $l_1$ & $l_1$/$l_2$ & $l_1$ - $l_2$ \\
\hline
Fraction nonzero & 0.4752 & 0.2683 & 0.2645 & 0.2677 \\
Sum of squares error& 1111.2 & 19107 & 1395.3 & 1335.6 \\
\hline
\end{tabular}
\caption{Fraction of nonzero abundances and sum of squares error for four demixing models} \label{density_sos}
\end{center}
\end{table}
The $l_1$ penalty promotes sparse solutions by trying to move coefficient vectors perpendicular to the positive face of the $l_1$ ball, shrinking the magnitudes of all elements.  The $l_1$/$l_2$ penalty, and to some extent $l_1$ - $l_2$, promote sparsity by trying to move in a different direction, tangent to the $l_2$ ball.  They do a better job of preserving the magnitudes of the abundances while enforcing a similarly sparse solution.  This is reflected in their lower sum of squares errors.

The results of these demixing algorithms are also represented in Figure \ref{fraction_planes} as fraction planes, which are the rows of the abundance matrix visualized as images.  They show the spatial abundance of each endmember.
\begin{figure}
\begin{center}
\begin{tabular}{cc}
NNLS & $l_1$ \\
\includegraphics[width=0.5\textwidth]{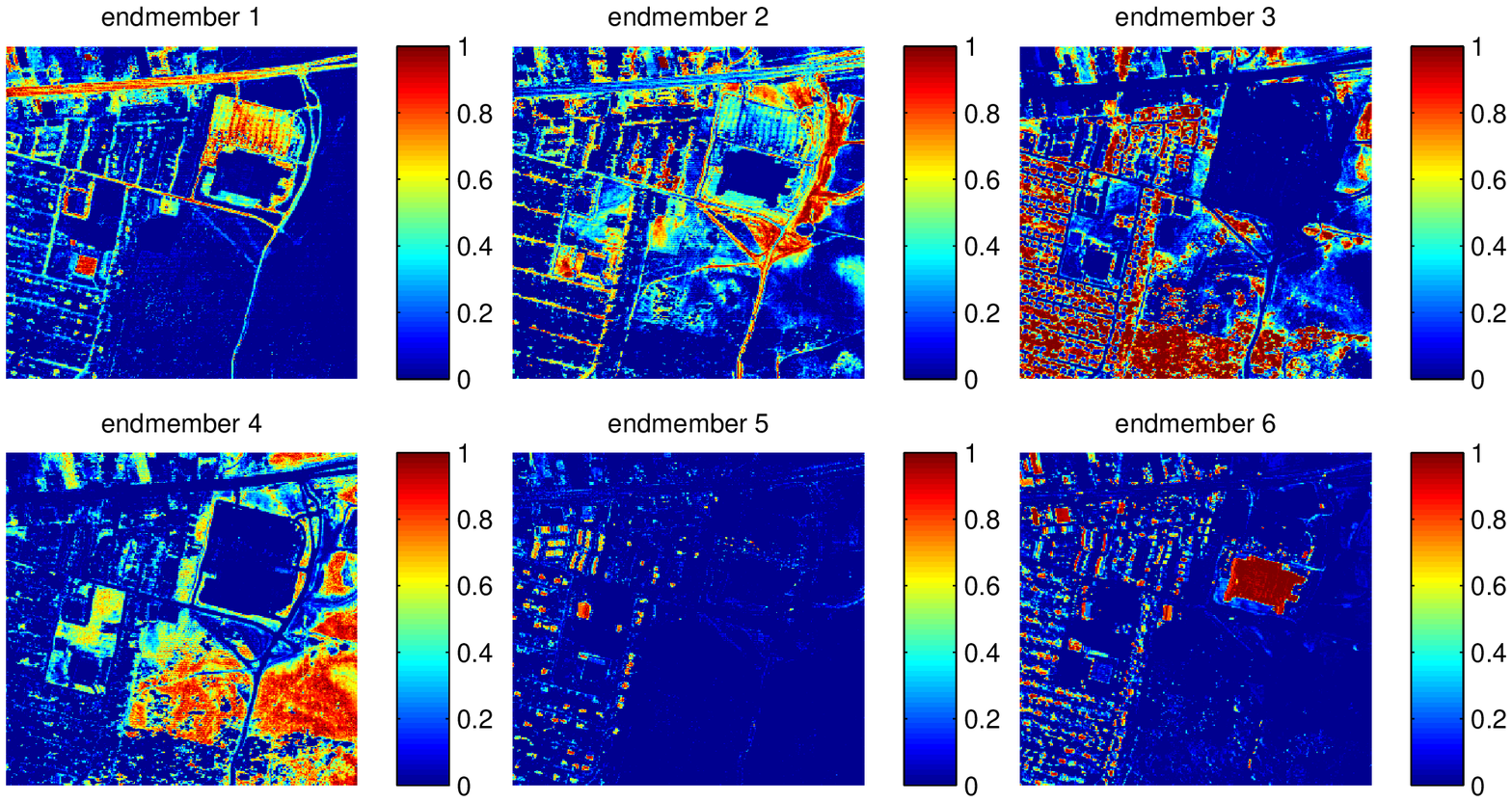}&
\includegraphics[width=0.5\textwidth]{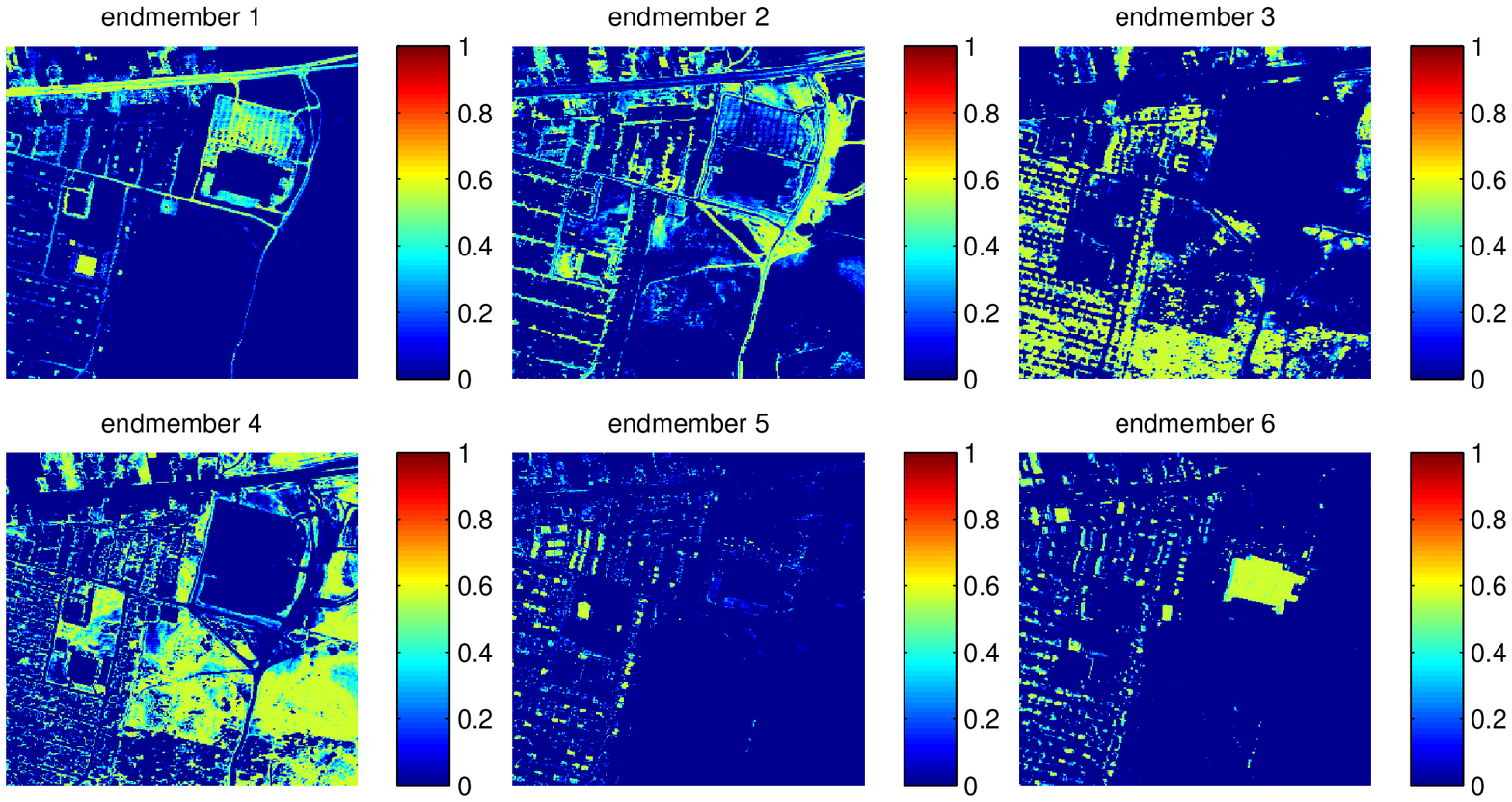}\\
$l_1$/$l_2$ & $l_1$ - $l_2$ \\
\includegraphics[width=0.5\textwidth]{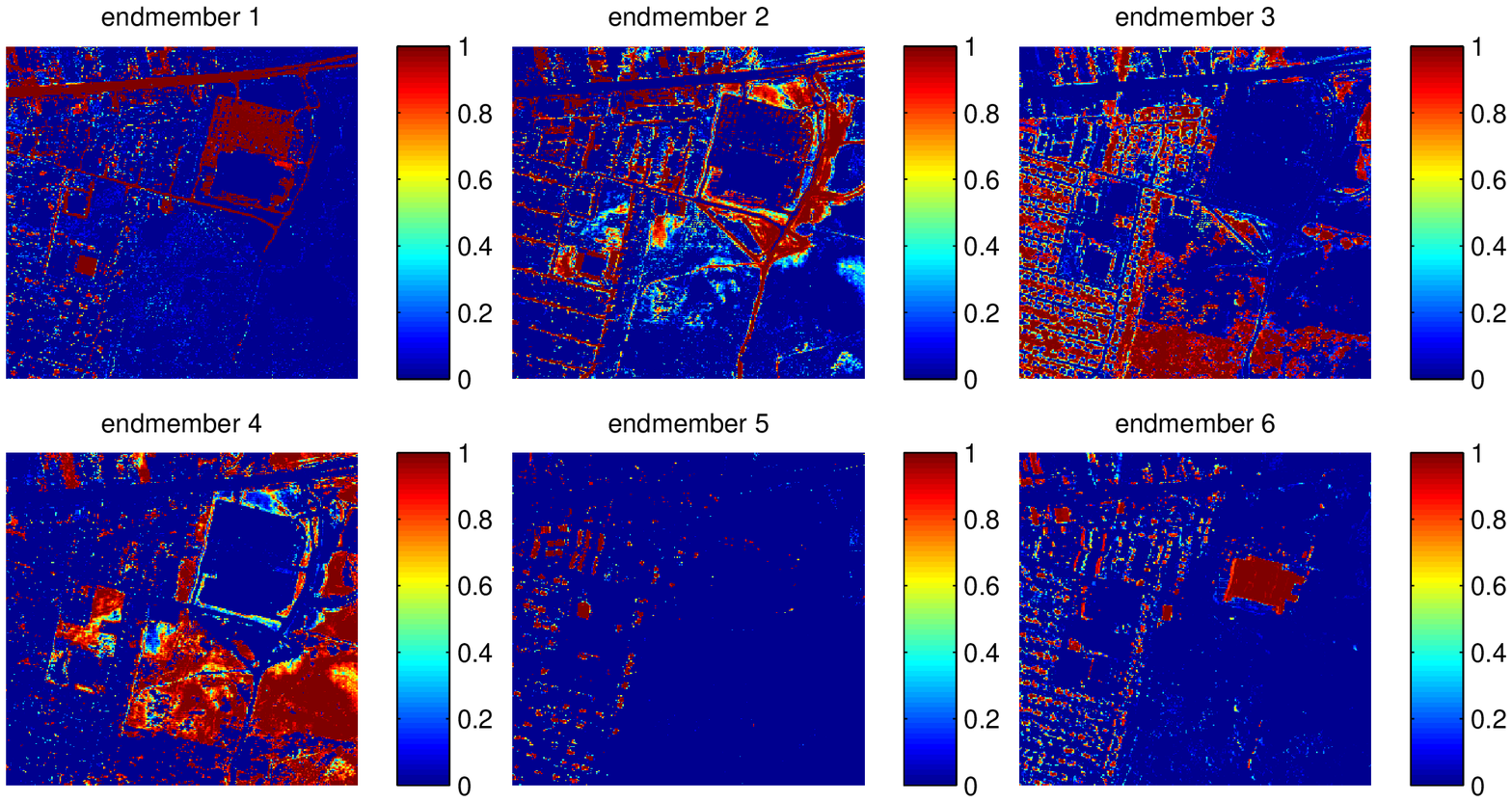}&
\includegraphics[width=0.5\textwidth]{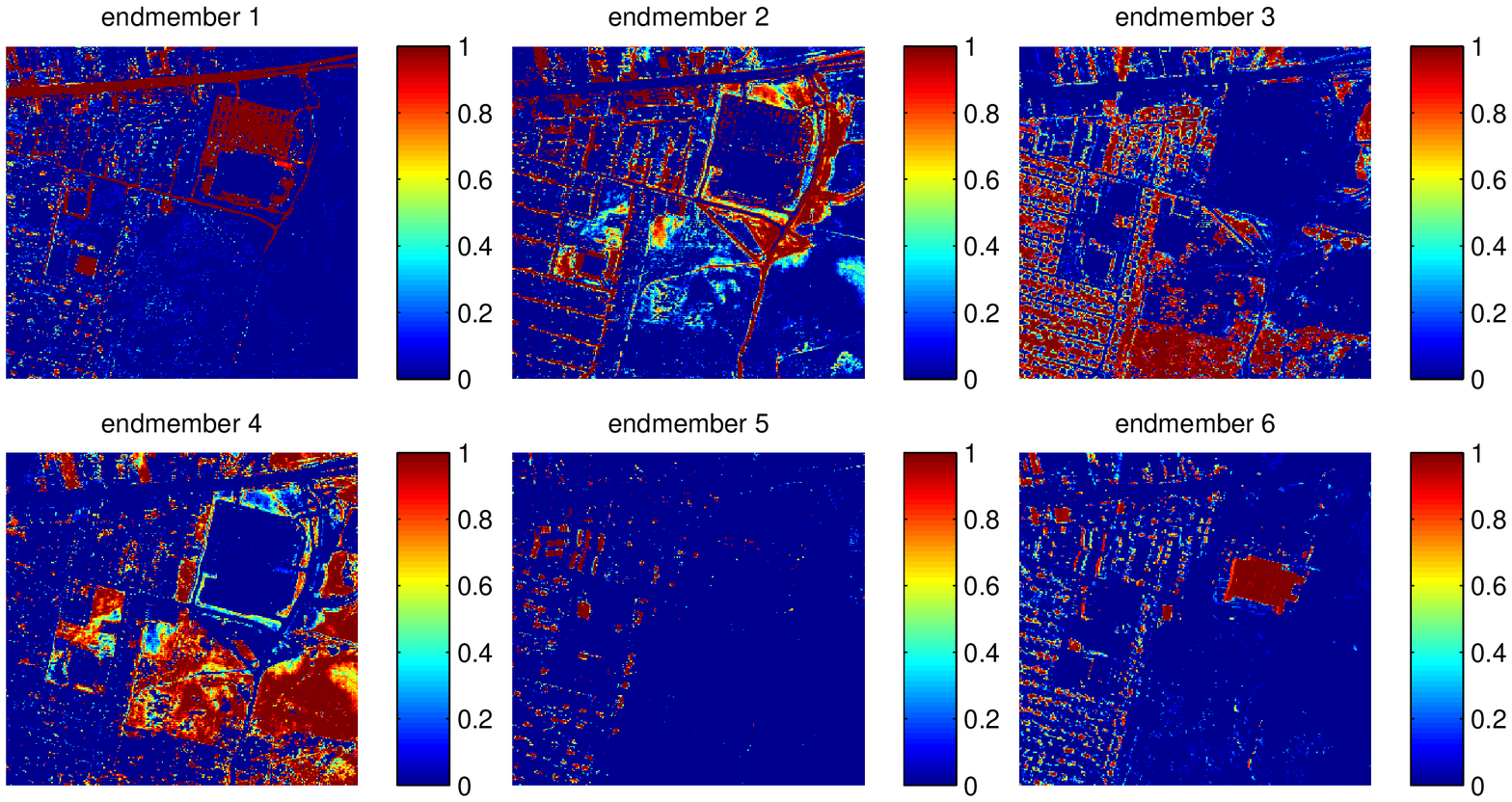}\\
\end{tabular}
\caption{Estimated fraction planes for urban data using hand selected endmembers } \label{fraction_planes}
\end{center}
\end{figure}

\subsection{Hyperspectral Demixing with Intra and Inter Sparsity Penalties}
In this Section we consider a hyperspectral demixing example with an expanded dictionary consisting of groups of references, each group consisting of candidate endmembers for a particular material.  The data we use for this examples is from \cite{salinasdata} and consists of a 204 band hyperspectral image of crops, soils and vineyards in Salinas Valley, California. Using a given ground truth labeling, we extract just the data corresponding to romaine lettuce at 4, 5, 6 and 7 weeks respectively.  For each of these four groups, we remove outliers and then randomly extract 100 representative signatures.  These and their normalized averages are plotted in Figure \ref{lettuce_100} and give a sense of the variability of the signatures corresponding to a particular label.

By concatenating the four groups of 100 signatures we construct a dictionary $A_{\text{group}} \in \R^{204 \times 400}$.  We also construct two smaller dictionaries $A_{\text{mean}} \text{ and } A_{\text{bad}} \in \R^{204 \times 4}$.  The columns of $A_{\text{mean}}$ are the average spectral signatures shown in red in Figure \ref{lettuce_100} and the columns of $A_{\text{bad}}$ are the candidate signatures farthest from the average shown in green in Figure \ref{lettuce_100}.

\begin{figure}
\begin{center}
\includegraphics[width=.9\textwidth]{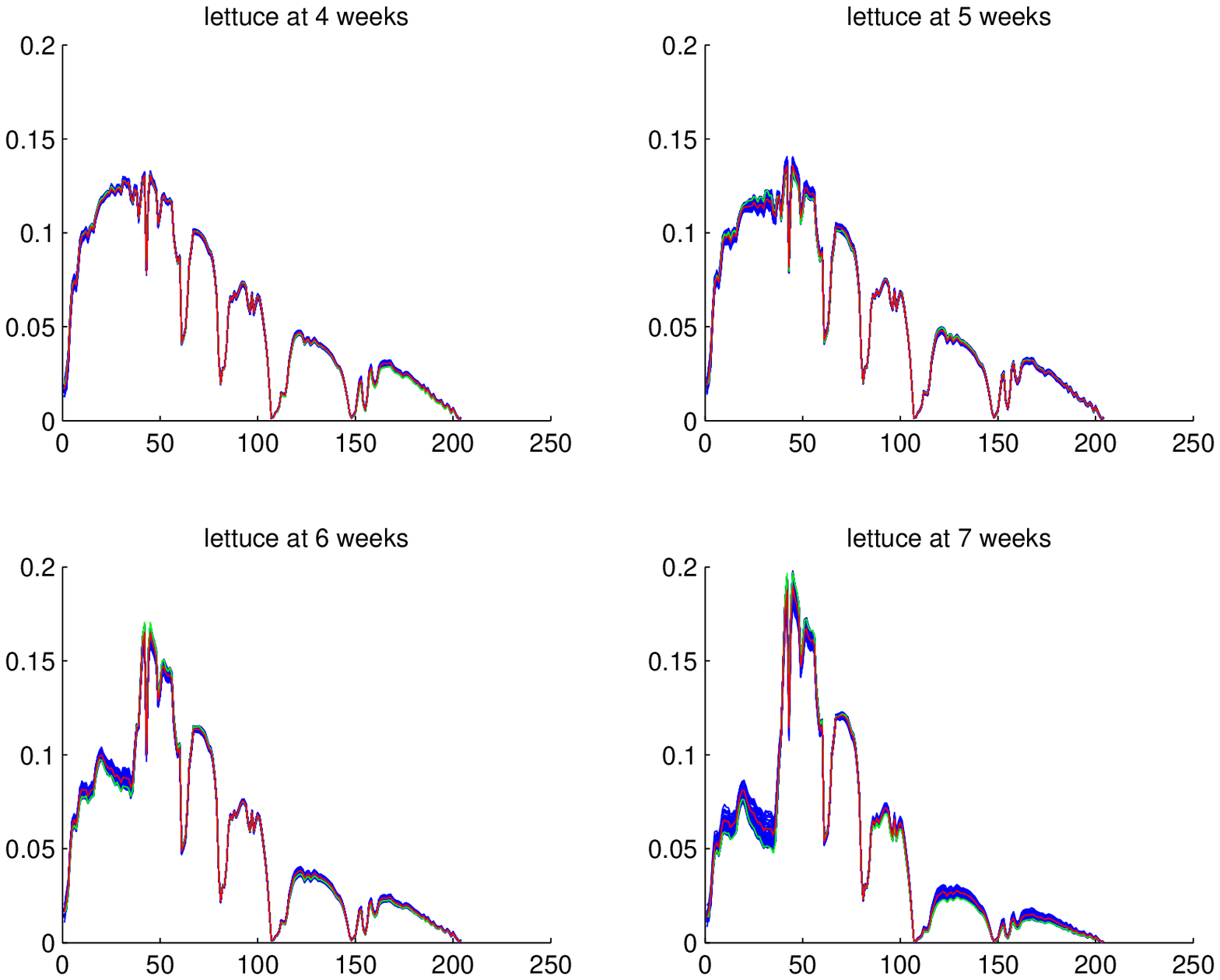} \\
\caption{Candidate endmembers (blue) for romaine lettuce at 4,5,6 and 7 weeks from Salinas dataset, normalized averages (red) and candidate endmembers farthest from the average (green) } \label{lettuce_100}
\end{center}
\end{figure}

Synthetic data $b \in \R^{204 \times 1560}$ was constructed by randomly constructing a ground truth abundance matrix $\bar{S}_{\text{group}} \in \R^{400 \times 1560}$ with 1000 1-sparse columns, 500 2-sparse columns, 50 3-sparse columns and 10 4-sparse columns, with each group of 100 coefficients being at most 1-sparse.  Zero mean Gaussian noise $\eta$ with standard deviation .005 was also added so that
\[b = A_{\text{group}}\bar{S}_{\text{group}} + \eta \ . \]
Each k-sparse abundance column was constructed by first randomly choosing k groups, then randomly choosing one element within each of the selected groups and assigning a random magnitude in $[0,1]$.  The generated columns were then rescaled so that the columns of the noise free data matrix would have unit $l_2$ norm.

Define $T \in \R^{4 \times 400}$ to be a block diagonal matrix with 1 by 100 row vectors of ones as the blocks.  
\[ T = \bbm 1 \cdots 1 & & & \\  & 1 \cdots 1 & & \\ & & 1 \cdots 1 & \\ & & & 1 \cdots 1 \ebm \ . \]
Applying $T$ to $\bar{S}_{\text{group}}$ lets us construct a ground truth group abundance matrix $\bar{S} \in \R^{4 \times 1560}$ by summing the adundances within groups.  For comparison purposes, this will allow us to apply different demixing methods using the different sized dictionaries $A_{\text{mean}}$, $A_{\text{group}}$ and $A_{\text{bad}}$ to compute $S_{\text{mean}}$, $TS_{\text{group}}$ and $S_{\text{bad}}$ respectively, which can all then be compared to $\bar{S}$.

We compare six different demixing methods using the three dictionaries:
\ben
\item NNLS (\ref{eq:NNLS}) using $A_{\text{mean}}$, $A_{\text{group}}$ and $A_{\text{bad}}$
\item $l_1$ (\ref{eq:l1demix}) using $A_{\text{mean}}$, $A_{\text{group}}$ and $A_{\text{bad}}$
\item $l_1$/$l_2$ (Problem 1) inter sparsity only, using $A_{\text{mean}}$ and $A_{\text{bad}}$
\item $l_1$ - $l_2$ (Problem 2) inter sparsity only, using $A_{\text{mean}}$ and $A_{\text{bad}}$
\item $l_1$/$l_2$ intra and inter sparsity, using $A_{\text{group}}$
\item $l_1$ - $l_2$ intra and inter sparsity, using $A_{\text{group}}$
\een
For $l_1$ demixing, we set $\gamma = .1$ for $A_{\text{mean}}$ and $A_{\text{bad}}$ and $\gamma = .001$ for $A_{\text{group}}$.  In all applications of Algorithms \ref{SGPalg_dynamic} and \ref{SGPalg}, we use a constant but nonzero initialization and set $\epsilon_j =.01$, $\gamma_0 = .01$ and $C = 10^{-9} \I$.  For the applications with intra sparsity penalties, $\gamma_j = .0001$ for $j = 1,2,3,4$.  Otherwise $\gamma_j = 0$.  For Algorithm \ref{SGPalg_dynamic}, we again use $\sigma = .1$, $\xi_1 = 2$ and $\xi_2 = 10$.  We stop iterating when the difference in the objective is less than $.001$.

We compare the computed group abundances to the ground truth $\bar{S}$ in two ways in Table \ref{Sgt}.  Measuring the $l_0$ norm of the difference of abundance matrices indicates how accurately the sparsity pattern was estimated.  For each material, we also compute the absolute value of each group abundance error averaged over all measurements.  For visualization, we plot the computed number of nonzero entries versus the ground truth for each column of the group abundances in Figure \ref{nnz_100}.

\begin{table}
\begin{center}
\begin{tabular}{|c|c|c|c|c|}
\hline
\ & NNLS & $l_1$ & $l_1$/$l_2$ & $l_1$ - $l_2$ \\
\hline
$\|S_{\text{mean}}-\bar{S}\|_0$ & 1537 & 935 & 786 & 784  \\
$E^{\text{mean}}_1$ & 0.0745 & 0.1355 & 0.0599 & 0.0591  \\
$E^{\text{mean}}_2$ & 0.0981 & 0.1418 & 0.0729 & 0.0722  \\
$E^{\text{mean}}_3$ & 0.0945 & 0.1627 & 0.0865 & 0.0868  \\
$E^{\text{mean}}_4$ & 0.0542 & 0.1293 & 0.0514 & 0.0492  \\
\hline
$\|TS_{\text{group}}-\bar{S}\|_0$ & 1814 & 851 & 889 & 851  \\
$E^{\text{group}}_1$ & 0.0624 & 0.1280 & 0.0717 & 0.0691  \\
$E^{\text{group}}_2$ & 0.0926 & 0.1300 & 0.0877 & 0.0782  \\
$E^{\text{group}}_3$ & 0.1066 & 0.1625 & 0.1147 & 0.1049  \\
$E^{\text{group}}_4$ & 0.0618 & 0.1249 & 0.0681 & 0.0621  \\
\hline
$\|S_{\text{bad}}-\bar{S}\|_0$ & 2123 & 1093 & 1134 & 1076 \\
$E^{\text{bad}}_1$ & 0.0804 & 0.1391 & 0.0666 & 0.0633  \\
$E^{\text{bad}}_2$ & 0.1410 & 0.1353 & 0.0900 & 0.0768  \\
$E^{\text{bad}}_3$ & 0.1400 & 0.1733 & 0.1000 & 0.1046  \\
$E^{\text{bad}}_4$ & 0.0759 & 0.1540 & 0.0646 & 0.0713  \\
\hline
\end{tabular}
\caption{Errors between computed group abundance and ground truth $\bar{S}$, where $E^{\text{mean}}_j = \frac{1}{P} \sum_{p=1}^P |S_{\text{mean}}(j,p)-\bar{S}(j,p)|$,
$E^{\text{group}}_j = \frac{1}{P} \sum_{p=1}^P |(TS_{\text{group}})(j,p)-\bar{S}(j,p)|$,
$E^{\text{bad}}_j = \frac{1}{P} \sum_{p=1}^P |S_{\text{bad}}(j,p)-\bar{S}(j,p)|$, } \label{Sgt}
\end{center}
\end{table}

\begin{figure}
\begin{center}
\includegraphics[width=\textwidth]{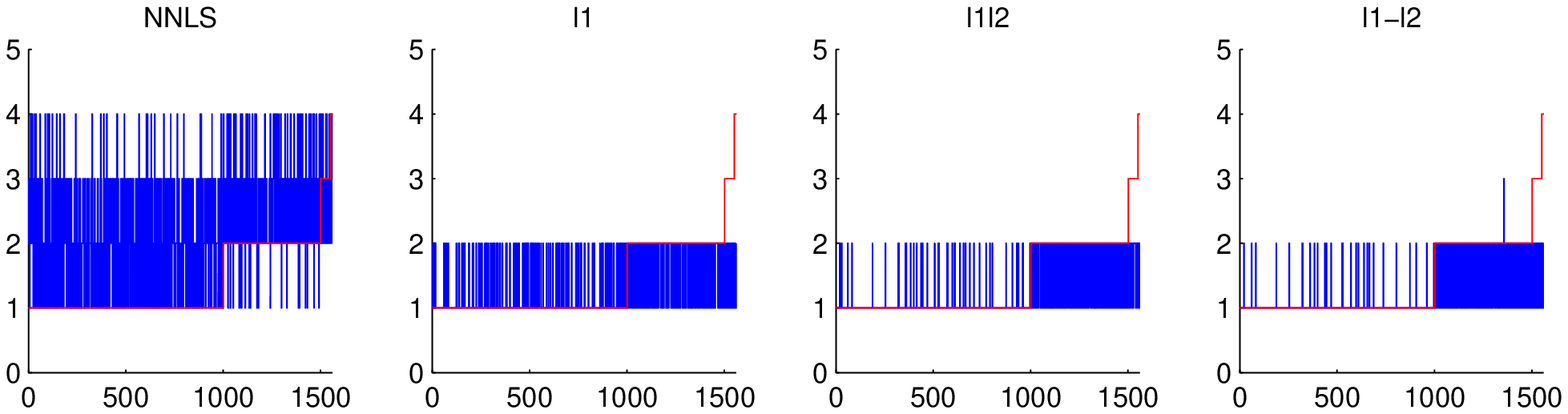} \\
\includegraphics[width=\textwidth]{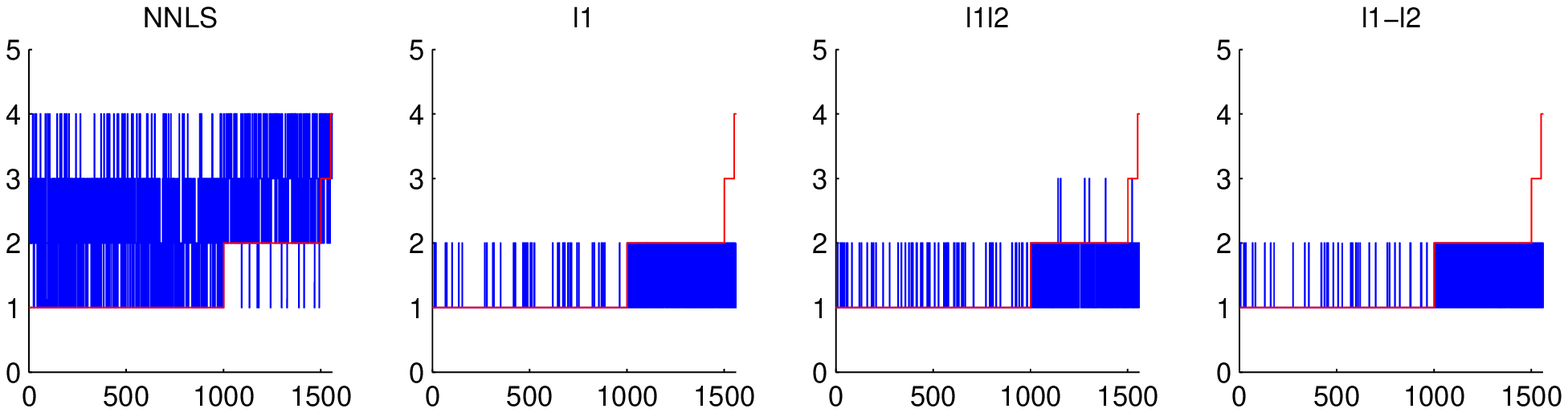} \\
\includegraphics[width=\textwidth]{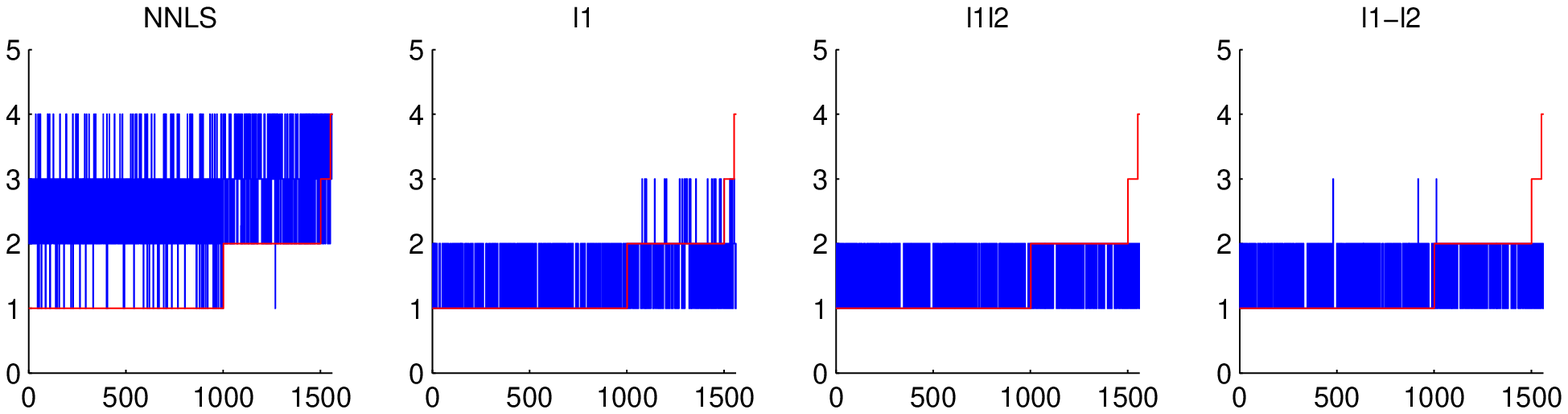} \\
\caption{Estimated number of nonzero entries in each abundance column (blue) and ground truth (red). \ Row 1: $S_{\text{mean}}$. \ Row 2: $TS_{\text{group}}$. \ Row 3: $S_{\text{bad}}$. } \label{nnz_100}
\end{center}
\end{figure}

We see in Table \ref{Sgt} and Figure \ref{nnz_100} that NNLS did a poor job at finding sparse solutions although average coefficient errors were low.  On the other hand, $l_1$ minimization did a good job of finding a sparse solution, but coefficient errors were higher because the abundance magnitudes were underestimated.  The $\frac{l_1}{l_2}$ and $l_1 - l_2$ minimization approaches were better at encouraging sparse solutions while maintaining small average errors in the abundance coefficients.

For this example, the average signatures used in $A_{\text{mean}}$ turned out to be good choices for the endmembers, and we didn't see any improvement in the estimated group abundances by considering the expanded dictionary $A_{\text{group}}$.  However, compared to using the four poorly selected endmember candidates in $A_{\text{bad}}$, we got better results with the expanded dictionary.  In the expanded dictionary case, which resulted in an underdetermined dictionary matrix, the abundances $S_{\text{group}}$ directly computed by $l_1$ minimization were much less sparse than those computed by $\frac{l_1}{l_2}$ and $l_1 - l_2$ minimization.  This is because $\frac{l_1}{l_2}$ and $l_1 - l_2$ minimization were able to enforce 1-sparsity within coefficient groups, but $l_1$ was not.  If the group 1-sparsity requirement is important for the model to be accurate, then this is an advantage of using the $\frac{l_1}{l_2}$ and $l_1 - l_2$ penalties.  Here, this difference in sparsity turned out not to have much effect on the group abundances $TS_{\text{group}}$, which were computed by summing the abundances within each group. This may not hold in situations where the endmember variability is more nonlinear.  For example, if the endmember variability had to do with misalignment, as with the earlier DOAS example, then linear combinations of misaligned signatures would not produce a good reference signature.



%% file: conclusions.tex
\section{Conclusions and Future Work \label{conclusions}}
We proposed a method for linear demixing problems where the dictionary contains multiple references for each material and we want to collaboratively choose the best one for each material present.  More generally, we showed how to use $\frac{l_1}{l_2}$ and $l_1 - l_2$ penalties to obtain structured sparse solutions to non-negative least squares problems.  These were reformulated as constrained minimization problems with differentiable but non-convex objectives.  A scaled gradient projection method based on difference of convex programming was proposed.  This approach requires solving a sequence of strongly quadratic programs, and we showed how these can be efficiently solved using the alternating direction method of multipliers.  Moreover, few iterations were required in practice, between 4 and 20 for all the numerical examples presented in this paper.  Some convergence analysis was also presented to show that limit points of the iterates are stationary points.  Numerical results for demixing problems in differential optical absorption spectroscopy and hyperspectral image analysis show that our difference of convex approach using $\frac{l_1}{l_2}$ and $l_1 - l_2$ penalties is capable of promoting different levels of sparsity on possibly overlapping subsets of the fitting or abundance coefficients.  

For future work we would like to test this method on more general multiple choice quadratic knapsack problems, which are related to the applications presented here that focused on finding solutions that were at most 1-sparse within specified groups.  It would be interesting to see how this variational approach performs relative to combinatorial optimization strategies for similar problems.  We are also interested in exploring alternative sparsity penalties that can be adapted to the data set.  When promoting 1-sparse solutions, the experiments in this paper used fixed sparsity parameters that were simply chosen to be sufficiently large.  We are interested in justifying the technique of gradually increasing this parameter while iterating, which empirically seems better able to avoid bad local minima.  The applications presented here all involved uncertainty in the dictionary, which was expanded to include multiple candidate references for each material.  If a-priori assumptions are available about the relative likelihood of these candidates, we would like to incorporate this into the model. \\


\noindent
{\bf Acknowledgments - }
We thank Lisa Wingen for providing DOAS references and data, which we used as a guide when generating synthetic data for some of our numerical examples.  We also thank John Greer for pointing out a paper by A. Zare and P. Gader \cite{ZG}. 

%% file: sparse_nnls.bbl
\begin{thebibliography}{10}

\bibitem{KM}
N.~Keshava and J.F. Mustard,
\newblock ``Spectral unmixing,''
\newblock {\em Signal Processing Magazine, IEEE}, vol. 19, no. 1, pp. 44 --57,
  jan 2002.

\bibitem{LH}
C.L. Lawson and R.J. Hanson,
\newblock {\em Solving Least Squares Problems},
\newblock Prentice-Hall: Englewood Cliffs, NJ, 1974.

\bibitem{SH}
M.~Slawski and M.~Hein,
\newblock ``Sparse recovery by thresholded non-negative least squares,''
\newblock in {\em NIPS}, 2011.

\bibitem{CDS}
Scott~Shaobing Chen, David~L. Donoho, and Michael~A. Saunders,
\newblock ``Atomic decomposition by basis pursuit,''
\newblock {\em SIAM Journal on Scientific Computing}, vol. 20, pp. 33--61,
  1998.

\bibitem{Tibshirani}
Robert Tibshirani,
\newblock ``Regression shrinkage and selection via the lasso,''
\newblock {\em Journal of the Royal Statistical Society. Series B
  (Methodological)}, vol. 58, no. 1, pp. pp. 267--288, 1996.

\bibitem{YOGD}
Wotao Yin, Stanley Osher, Donald Goldfarb, and Jerome Darbon,
\newblock ``Bregman iterative algorithms for l1minimization with applications
  to compressed sensing,''
\newblock {\em SIAM J. Imaging Sci}, vol. 1, pp. 143--168, 2008.

\bibitem{MGB}
L.~Meier, S.~van~de Geer, and P.~Buhlmann,
\newblock ``The group lasso for logistic regression,''
\newblock {\em J. R. Statist. Soc. B}, vol. 70, no. 1, pp. 53--71, 2008.

\bibitem{JAB}
R.~Jenatton, J.-Y. Audibert, and F.~Bach,
\newblock ``Structured variable selection with sparsity-inducing norms,''
\newblock Tech. {R}ep., 2010,
\newblock arXiv:0904.3523v3.

\bibitem{QG}
Z.~Qin and D.~Goldfarb,
\newblock ``Structured sparsity via alternating direction methods,''
\newblock {\em The Journal of Machine Learning Research}, vol. 98888, pp.
  1435--1468, 2012.

\bibitem{Pauca06}
V.~P. Pauca, J.~Piper, and R.~J. Plemmons,
\newblock ``Nonnegative matrix factorization for spectral data analysis,''
\newblock {\em Linear Algebra and its Applications}, vol. 416, no. 1, pp.
  29--47, 2006.

\bibitem{MBPS}
Julien Mairal, Francis Bach, Jean Ponce, and Guillermo Sapiro,
\newblock ``Online dictionary learning for sparse coding,''
\newblock in {\em Proceedings of the 26th Annual International Conference on
  Machine Learning}, New York, NY, USA, 2009, ICML '09, pp. 689--696, ACM.

\bibitem{H2002}
P.~O. Hoyer,
\newblock ``Non-negative sparse coding,''
\newblock in {\em Proc. IEEE Workshop on Neural Networks for Signal
  Processing}, 2002,
\newblock pages 557-565.

\bibitem{LS}
D.~D. Lee and H.~S. Seung,
\newblock ``Algorithms for non-negative matrix factorization,''
\newblock in {\em Proc. NIPS, Advances in Neural Information Processing 13},
  2001.

\bibitem{BBLPP}
M.~Berry, M.~Browne, A.~Langville, P.~Pauca, and R.J. Plemmons,
\newblock ``Algorithms and applications for approximate nonnegative matrix
  factorization,''
\newblock {\em Computational Statistics and Data Analysis}, vol. 52, pp.
  155--173, 2007.

\bibitem{EMOSX}
E.~Esser, M.~Moller, S.~Osher, G.~Sapiro, and J.~Xin,
\newblock ``A convex model for nonnegative matrix factorization and
  dimensionality reduction on physical space,''
\newblock {\em IEEE Trans. Imag. Proc.}, vol. 21, no. 7, 2012.

\bibitem{PS}
U.~Platt and J.~Stutz,
\newblock {\em Differential Optical Absorption Spectroscopy: Principles and
  Applications},
\newblock Springer, 2008.

\bibitem{unmixing}
J.~M. Bioucas-Dias, A.~Plaza, N.~Dobigeon, M.~Parente, Q.~Du, P.~Gader, and
  J.~Chanussot,
\newblock ``Hyperspectral unmixing overview: Geometrical, statistical, and
  sparse regression-based approaches,''
\newblock {\em arXiv:1202.6294v2}, 2012.

\bibitem{Greer10}
J.~Greer,
\newblock ``Sparse demixing,''
\newblock {\em SPIE proceedings on Algorithms and Technologies for
  Multispectral, Hyperspectral, and Ultraspectral Imagery XVI}, vol. 7695, pp.
  76951O--76951O--12, 2010.

\bibitem{GWO}
Z.~Guo, T.~Wittman, and S.~Osher,
\newblock ``{L}1 unmixing and its application to hyperspectral image
  enhancement,''
\newblock {\em Proceedings SPIE Conference on Algorithms and Technologies for
  Multispectral, Hyperspectral, and Ultraspectral Imagery XV}, vol. 7334, pp.
  73341M--73341M--9, 2008.

\bibitem{HLS}
Y.~H. Hu, H.~B. Lee, and F.~L. Scarpace,
\newblock ``Optimal linear spectral unmixing,''
\newblock {\em Geoscience and Remote Sensing, IEEE Transactions on}, vol. 37,
  no. 1, pp. 639 --644, jan 1999.

\bibitem{ZG}
A.~Zare and P.~Gader,
\newblock ``{PCE}: Piece-wise convex endmember detection,''
\newblock {\em IEEE Transactions on Geoscience and Remote Sensing}, vol. 48,
  no. 6, pp. 2620--2632, 2010.

\bibitem{louBS11}
Y.~Lou, A.~L. Bertozzi, and S.~Soatto,
\newblock ``Direct sparse deblurring,''
\newblock {\em Journal of Mathematical Imaging and Vision}, vol. 39, no. 1, pp.
  1--12, 2011.

\bibitem{JOB}
R.~Jenatton, G.~Obozinski, and F.~Bach,
\newblock ``Structured sparse principal component analysis,''
\newblock in {\em International Conference on Artificial Intelligence and
  Statistics (AISTATS)}, 2010.

\bibitem{BEZ}
A.~M. Bruckstein, M.~Elad, and M.~Zibulevsky,
\newblock ``On the uniqueness of nonnegative sparse solutions to
  underdetermined systems of equations,''
\newblock {\em Information Theory, IEEE Transactions on}, vol. 54, no. 11, pp.
  4813 --4820, nov. 2008.

\bibitem{Tropp}
J.~A. Tropp,
\newblock ``Greed is good: algorithmic results for sparse approximation,''
\newblock {\em Information Theory, IEEE Transactions on}, vol. 50, no. 10, pp.
  2231 -- 2242, oct. 2004.

\bibitem{CRT}
E.~Candes, J.~Romberg, and T.~Tao,
\newblock ``Stable signal recovery from incomplete and inaccurate
  measurements,''
\newblock {\em Comm. Pure Appl. Math.}, vol. 59, pp. 1207--1223, 2006.

\bibitem{IBP}
M.~D. Iordache, J.~M. Bioucas-Dias, and A.~Plaza,
\newblock ``Sparse unmixing of hyperspectral data,''
\newblock {\em Geoscience and Remote Sensing, IEEE Transactions on}, vol. 49,
  no. 6, pp. 2014 --2039, june 2011.

\bibitem{luZ12}
Zhaosong Lu and Yong Zhang,
\newblock ``Penalty decomposition methods for {L0}-norm minimization,''
\newblock {\em arXiv:1008.5372v2}, 2012.

\bibitem{H2004}
P.~O. Hoyer,
\newblock ``Non-negative matrix factorization with sparseness constraints,''
\newblock {\em Journal of Machine Learning Research}, vol. 5, no. 12, pp.
  1457--1469, 2004.

\bibitem{HR}
N.~Hurley and S.~Rickard,
\newblock ``Comparing measures of sprasity,''
\newblock {\em IEEE Transactions on Information Theory}, vol. 55, no. 10, pp.
  4723--4741, 2009.

\bibitem{KTF}
D.~Krishnan, T.~Tay, and R.~Fergus,
\newblock ``Blind deconvolution using a normalized sparsity measure,''
\newblock in {\em CVPR}, 2011.

\bibitem{JLSW}
H.~Ji, J.~Li, Z.~Shen, and K.~Wang,
\newblock ``Image deconvolution using a characterization of sharp images in
  wavelet domain,''
\newblock {\em Applied and Computational Harmonic Analysis}, vol. 32, no. 2,
  pp. 295--304, 2012.

\bibitem{TA}
Pham~Dinh Tao and Le~Thi~Hoai An,
\newblock ``Convex analysis approach to d.c. programming: Theory, algorithms
  and applications,''
\newblock {\em Acta Mathematica Vietnamica}, vol. 22, no. 1, pp. 289--355,
  1997.

\bibitem{E1998}
D.~Eyre,
\newblock ``An unconditionally stable one-step scheme for gradient systems,''
  1998,
\newblock www.math.utah.edu/~eyre/research/methods/stable.ps.

\bibitem{VR}
B.~P. Vollmayr-Lee and A.~D. Rutenberg,
\newblock ``Fast and accurate coarsening simulation with an unconditionally
  stable time step,''
\newblock {\em Phys. Rev. E}, vol. 68, no. 6, 2003.

\bibitem{BEG}
A.~L. Bertozzi, S.~Esedoglu, and A.~Gillette,
\newblock ``Analysis of a two-scale {Cahn-Hilliard} model for image
  inpainting,''
\newblock {\em Multiscale Modeling and Simulation}, vol. 6, no. 3, pp.
  913--936, 2007.

\bibitem{BT}
D.~Bertsekas and J.~Tsitsiklis,
\newblock {\em Parallel and Distributed Computation},
\newblock Prentice Hall, 1989.

\bibitem{B}
D.~Bertsekas,
\newblock {\em Nonlinear Programming},
\newblock Athena Scientific, 1999.

\bibitem{BZZ}
S.~Bonettini, R.~Zanella, and L.~Zanni,
\newblock ``A scaled gradient projection method for constrained image
  deblurring,''
\newblock {\em Inverse problems}, vol. 25, 2009.

\bibitem{LHY}
K.~Lange, D.~Hunter, and I.~Yang,
\newblock ``Optimization transfer using surrogate objective functions,''
\newblock {\em Journal of Computational and Graphical Statistics}, vol. 9, no.
  1, pp. 1--20, 2000.

\bibitem{GM}
D.~Gabay and B.~Mercier,
\newblock ``A dual algorithm for the solution of nonlinear variational problems
  via finite-element approximations,''
\newblock {\em Comp. Math. Appl.}, vol. 2, pp. 17--40, 1976.

\bibitem{GlM}
R.~Glowinski and A.~Marrocco,
\newblock ``Sur l'approximation par elements finis d'ordre un, et la resolution
  par penalisation-dualite d'une classe de problemes de {D}irichlet
  nonlineaires,''
\newblock {\em Rev. Francaise d'Aut Inf. Rech. Oper.}, vol. R-2, pp. 41--76,
  1975.

\bibitem{GO}
T.~Goldstein and S.~Osher,
\newblock ``The split bregman method for l1-regularized problems,''
\newblock {\em SIAM Journal on Imaging Science}, vol. 2, no. 2, pp. 323--343,
  2009.

\bibitem{SGO}
A.~Szlam, Z.~Guo, and S.~Osher,
\newblock ``A split {B}regman method for non-negative sparsity penalized least
  squares with applications to hyperspectral demixing,''
\newblock in {\em International Conference on Image Processing (ICIP)}, 2010,
  pp. 1917--1920.

\bibitem{Brucker}
Peter Brucker,
\newblock ``An o(n) algorithm for quadratic knapsack problems,''
\newblock {\em Operations Research Letters}, vol. 3, no. 3, 1984.

\bibitem{SR}
N.~Saito and J-F. Remy,
\newblock ``The polyharmonic local sine transform: A new tool for local image
  analysis and synthesis without edge effect,''
\newblock {\em Applied and Computational Harmonic Analysis}, vol. 20, no. 1,
  pp. 41--73, 2006.

\bibitem{winter99}
M.~E. Winter,
\newblock ``{N-FINDR}: an algorithm for fast autonomous spectral end-member
  determination in hyperspectral data,''
\newblock in {\em Imaging Spectrometry V}. 1999, vol. 3753, pp. 266--275, SPIE.

\bibitem{Nascimento05}
J.~M.~P. Nascimento and J.~M. Bioucas-Dias,
\newblock ``Vertex component analysis: A fast algorithm to unmix hyperspectral
  data,''
\newblock {\em IEEE Trans. Geosci. Rem. Sens.}, vol. 43, pp. 898--910, 2004.

\bibitem{Zare08}
A.~Zare,
\newblock ``Hyperspectral endmember detection and band selection using
  {B}ayesian methods,'' 2008.

\bibitem{Castrodad10}
A.~Castrodad, Z.~Xing, J.~Greer, E.~Bosch, L.~Carin, and G.~Sapiro,
\newblock ``Learning discriminative sparse models for source separation and
  mapping of hyperspectral imagery,'' Submitted September 2010,
  http://www.ima.umn.edu/preprints/oct2010/oct2010.html.

\bibitem{guoO11}
Z.~Guo and S.~Osher,
\newblock ``Template matching via l1 minimization and its application to
  hyperspectral data,''
\newblock {\em Inverse Problems Imaging}, vol. 1, no. 5, pp. 19--35, 2011.

\bibitem{DOASdata}
Finlayson-Pitts,
\newblock ``Unpublished data,'' Provided by Lisa Wingen, 2000.

\bibitem{urbdata}
``Urban dataset,'' US Army Corps of Engineers. Freely available online at
  www.tec.army.mil/hypercube.

\bibitem{salinasdata}
``Salinas aviris dataset,'' Available online at \newline
  http://www.ehu.es/ccwintco/index.php/Hyperspectral\_Remote\_Sensing\_Scenes\#Salinas.

\end{thebibliography}
